\documentclass[10pt]{article}
\usepackage[utf8]{inputenc}
\usepackage[margin=1in]{geometry}
\usepackage{rotating}
\usepackage{booktabs}
\usepackage{array}
\usepackage{amsmath, amsfonts}
\usepackage{color}
\usepackage{hyperref}
\usepackage{bbm}
\usepackage{algorithm}
\usepackage[noend]{algpseudocode}
\usepackage{subcaption}

\usepackage{mathtools}
\mathtoolsset{showonlyrefs}

\usepackage[
    backend=biber,
    style=ieee,
    citestyle=numeric-comp,
    sorting=none,
    giveninits=true,
    natbib,
    hyperref,
    maxbibnames=99,
    doi=true,isbn=false,url=true,eprint=false
]{biblatex}

\usepackage{amsthm}
\usepackage[normalem]{ulem}
\usepackage{soul}

\newtheorem{theorem}{Theorem} 
\newtheorem{prop}{Proposition}

\newtheorem{definition}{Definition}
\newtheorem{remark}{Remark}

\usepackage{chngcntr}
\usepackage{apptools}
\AtAppendix{\counterwithin{theorem}{section}}
\AtAppendix{\counterwithin{prop}{section}}

\usepackage{mathtools}
\DeclarePairedDelimiter{\ceil}{\lceil}{\rceil}

\usepackage[dvipsnames]{xcolor}

\def\Ical{\mathcal{I}}
\def\I{\mathbb{I}}

\def\R{\mathbb{R}}
\def\E{\mathbb{E}}

\def\Tau{\mathcal{T}}
\newcommand{\Y}{\mathcal{Y}}
\newcommand{\X}{\mathcal{X}}

\newcommand{\T}{\mathcal{T}}
\newcommand{\G}{\mathcal{G}}
\renewcommand{\S}{\mathcal{S}}

\newcommand{\lhat}{\hat{\lambda}}
\newcommand{\Rhat}{\widehat{R}}
\newcommand{\ind}[1]{\mathbbm{1}_{\left\{#1\right\}}}
\newcommand{\triangleq}{\stackrel{\triangle}{=}}

\makeatletter
\def\blfootnote{\xdef\@thefnmark{}\@footnotetext}
\makeatother

\title{Distribution-Free, Risk-Controlling Prediction Sets}
\author{Stephen Bates\thanks{equal contribution}, Anastasios Angelopoulos\footnotemark[1], Lihua Lei\footnotemark[1], Jitendra Malik, Michael I.~Jordan}
\date{\today}

\begin{document}

\maketitle

\begin{abstract}
While improving prediction accuracy has been the focus of machine learning in recent years, this alone does not suffice for reliable decision-making. 
Deploying learning systems in consequential settings also requires calibrating and communicating the uncertainty of predictions.
To convey instance-wise uncertainty for prediction tasks, we show how to generate set-valued predictions from a black-box predictor that control the expected loss on future test points at a user-specified level.
Our approach provides explicit finite-sample guarantees for any dataset by using a holdout set to calibrate the size of the prediction sets. 
This framework enables simple, distribution-free, rigorous error control for many tasks, and we demonstrate it in five large-scale machine learning problems: 
(1) classification problems where some mistakes are more costly than others; 
(2) multi-label classification, where each observation has multiple associated labels; 
(3) classification problems where the labels have a hierarchical structure; 
(4) image segmentation, where we wish to predict a set of pixels containing an object of interest; and
(5) protein structure prediction.
Lastly, we discuss extensions to uncertainty quantification for ranking, metric learning and distributionally robust learning.
\end{abstract}

\section{Introduction}
\label{sec:introduction}

Black-box predictive algorithms have begun to be deployed in many real-world decision-making settings. Problematically, however, these algorithms are rarely accompanied by reliable uncertainty quantification. Algorithm developers often depend on the standard training/validation/test paradigm to make assertions of accuracy, stopping short of any further attempt to indicate that an algorithm's predictions should be treated with skepticism. Thus, prediction failures will often be silent ones, which is particularly alarming in high-consequence settings. \blfootnote{see project website at \href{http://www.angelopoulos.ai/blog/posts/rcps/}{\textcolor{blue}{angelopoulos.ai/blog/posts/rcps/}}}

While one reasonable response to this problem involves retreating from black-box prediction, such a retreat raises many unresolved problems, and it is clear that black-box prediction will be with us for some time to come.  A second response is to modify black-box prediction procedures so that they provide reliable uncertainty quantification, thereby supporting a variety of post-prediction activities, including risk-sensitive decision-making, audits, and protocols for model improvement.

We introduce a method for modifying a black-box predictor to return a set of plausible responses that  limits the frequency of costly errors to a level chosen by the user. Returning a set of responses is a useful way to represent uncertainty, since such sets can be readily constructed from any existing predictor and, moreover, they are often interpretable. We call our proposed technique \emph{risk-controlling prediction sets} (RCPS).  The idea is to produce prediction sets that provide distribution-free, finite-sample control of a general loss.

As an example, consider classifying MRI images as in Figure~\ref{fig:mri_ex}.
Each image can be classified into one of several diagnostic categories. 
We encode the consequence (\emph{loss}) of making a mistake on an image as 100 for the most severe mistake (class \texttt{stroke}) and as 0.1 for the least severe mistake (class \texttt{normal}).
Our procedure returns a set of labels, such as those denoted by the red, blue, and green brackets in Figure~\ref{fig:mri_ex}. This output set represents the plausible range of patient diagnoses, accounting for their respective severities.  Our procedure returns sets that are guaranteed to keep average loss (\emph{risk}) on future data below a user-specified level, under a set of assumptions that we make explicit. To do this, the size of the output set is chosen based on the accuracy of the classifier and the desired risk level---a lower accuracy classifier or a more strict risk level will require larger sets to guarantee risk control.
Because of the explicit guarantee on our output, a doctor could safely exclude diagnoses outside the set and test for those within.

Formally, for a test point with features $X \in \X$, a response $Y \in \Y$, we consider set-valued predictors $\T(X) : \X \to \Y'$ where $\Y'$ is some space of sets; we take $\Y' = 2^\Y$ in the MRI above example and for most of this work. We then have a loss function on set-valued predictions $L: \Y \times \Y' \to \R$ that encodes our notion of consequence, and seek a predictor $\T$, that controls the risk $R(\T)=\E\big[L(Y,\T(X))]$.
For example, in our MRI setting, if the first argument is a label $y \notin \T(X)$, and the second argument is $\T(X)$, the loss function outputs the cost of {\em not} predicting $y$.
Our goal in this work is to create set-valued predictors from training data that have risk that is below some desired level $\alpha$, with high probability. Specifically, we seek the following:
\begin{definition}[Risk-controlling prediction sets]
Let $\T$ be a random function taking values in the space of functions $\X \to \Y'$ (e.g., a functional estimator trained on data). We say that $\T$ is a \emph{$(\alpha,\delta)$-risk-controlling prediction set} if, with probability at least $1-\delta$, we have $R(\T) \le \alpha$.
\label{def:rcps}
\end{definition}
\noindent The error level $(\alpha, \delta)$ is chosen in advance by the user. The reader should think of 10\% as a representative value of $\delta$; the choice of $\alpha$ will vary with the choice of loss function.

\begin{figure}[t]
    \centering
    \includegraphics[width=0.9\textwidth]{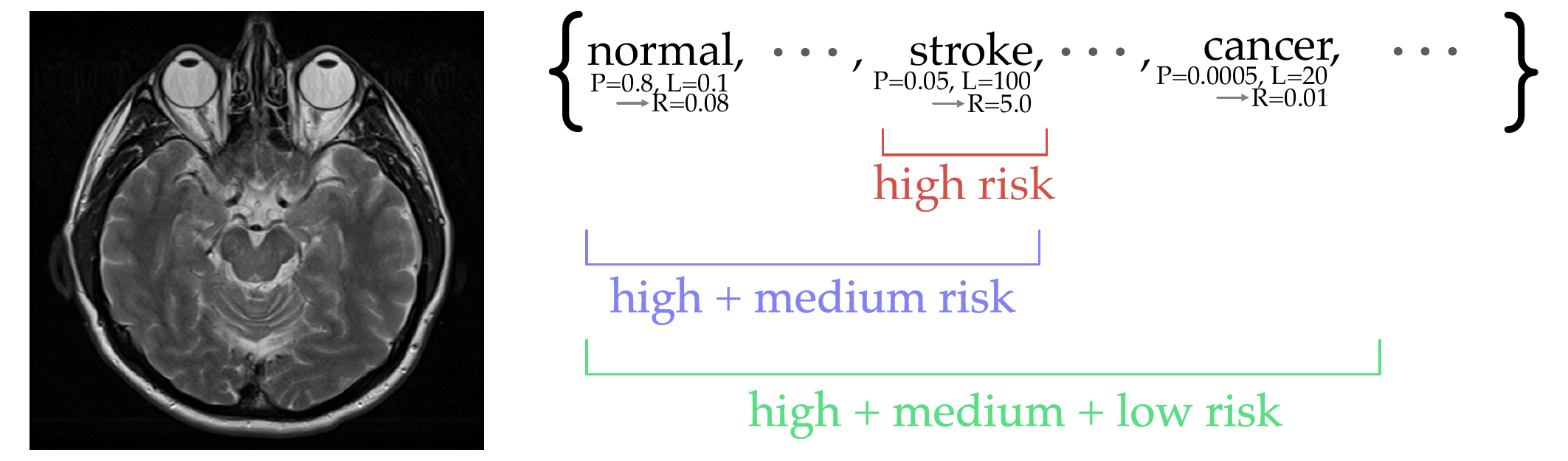}
    \caption{\textbf{A stylized example of risk-controlling prediction sets.} Here, ``P'' gives the estimated probability for each class, the loss per class is labeled as ``L,'' and the loss times the probability is the estimated risk, labeled as ``R.'' The red, blue, and green brackets represent possible sets of labels that our procedure may output.}
    \label{fig:mri_ex}
\end{figure}

\subsection*{Related work}
Prediction sets have a long history in statistics, going back at least to tolerance regions in the 1940s \cite{wilks1941, wilks1942, wald1943, tukey1947}. Tolerance regions are sets that contain a desired fraction of the population distribution with high probability. For example, one may ask for a region that contains 90\% of future test points with probability 99\% (over the training data). See \cite{krishnamoorthy2009statistical} for an overview of tolerance regions. Recently, tolerance regions have been instantiated to form prediction sets for deep learning models \cite{Park2020PAC, park2021pac}.
In parallel, conformal prediction \cite{vovk1999machine, vovk2005algorithmic} has been recognized as an attractive way of producing predictive sets with finite-sample guarantees. A particularly convenient form of conformal prediction, known as \emph{split conformal prediction} \cite{papadopoulos2002inductive, lei2013conformal}, uses data splitting to generate prediction sets in a computationally efficient way; see also \cite{vovk2015cross, barber2019jackknife} for generalizations that re-use data for improved  statistical efficiency. Conformal prediction is a generic approach, and much recent work has focused on designing specific conformal procedures to have good performance according to metrics such as small set sizes \cite{Sadinle2016LeastAS}, approximate coverage in all regions of feature space \cite{barber2019limits, romano2019conformalized, izbicki2019flexible, romano2020classification, cauchois2020knowing, guan2020conformal, angelopoulos2020sets}, and errors balanced across classes \cite{lei2014classification, Sadinle2016LeastAS, hechtlinger2018cautious, guan2019prediction}. Further extensions of conformal prediction address topics such as distribution estimation \cite{vovk2019conformal}, causal inference \cite{lei2020conformal}, and handling or testing distribution shift \cite{tibshirani2019conformal, cauchois2020robust, hu2020distributionfree}.
As an alternative to conformal prediction and tolerance regions, there is also a set of techniques that approach the tradeoff between small sets and high coverage by defining a utility function balancing these two considerations and finding the set-valued predictor that maximizes this utility \cite[e.g.,][]{grycko1993classification, delcoz2009learning, mortier2020efficient}. The present work concerns the construction of tolerance regions with a user-specified coverage guarantee, and we do not pursue this latter formulation here.

In the current work, we expand the notion of tolerance regions to apply to a wider class of losses for set-valued predictors. Our development is inspired by the nested set interpretation of conformal prediction articulated in \cite{gupta2020nested}, and our proposed algorithm is somewhat similar to split conformal prediction. Unlike conformal prediction, however, we pursue the high-probability error guarantees of tolerance regions and thus rely on entirely different proof techniques---see \cite{vovk2012conditional} for a discussion of their relationship. As one concrete instance of this framework, we introduce a family of set-valued predictors that generalizes those of \cite{Sadinle2016LeastAS} to produce small set-valued predictions in a wide range of settings.

\subsection*{Our contribution}

The central contribution of this work is a procedure to calibrate prediction sets to have finite-sample control of any loss satisfying a certain monotonicity requirement.
The calibration procedure applies to any set-valued predictor, but we also show how to take any standard (non-set-valued) predictor and turn it into a set-valued predictor that works well with our calibration procedure.
Our algorithm includes the construction of tolerance regions as special case, but applies to many other problems; this work explicitly considers classification with different penalties for different misclassification events, multi-label classification, classification with hierarchically structured classes, image segmentation, prediction problems where the response is a 3D structure, ranking, and metric learning.

\section{Upper Confidence Bound Calibration}

This section introduces our proposed method to calibrate any set-valued predictor so that it is guaranteed to have risk below a user-specified level, i.e., so that it satisfies Definition~\ref{def:rcps}.

\begin{figure}[t]
    \centering
    \includegraphics[width = 3in, trim = 0 0cm 0cm 0, clip]{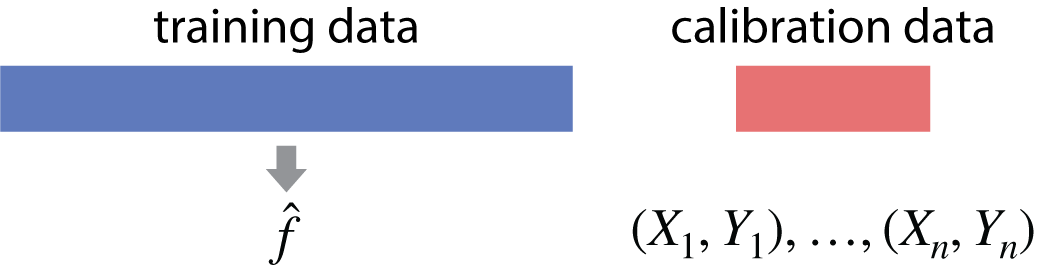}
    \caption{\textbf{Sample-splitting setup.} The training data is used to fit a predictive model $\hat{f}$. The remaining data is used to calibrate a set-valued predictor (based on $\hat{f}$) to control risk, as described in this work.}
    \label{fig:data_split}
\end{figure}

\subsection{Setting and notation}

Let $(X_i, Y_i)_{i = 1,\dots,m}$ be an independent and identically distributed (i.i.d.) set of variables, where the features vectors $X_i$ take values in $\X$ and the response $Y_i$ take values in $\Y$. To begin, split our data into a \emph{training set} and a \emph{calibration set}. Formally, let $\{\Ical_{\textnormal{train}}, \Ical_{\textnormal{cal}}\}$ form a partition of $\{1,\dots,m\}$, and let $n = |\Ical_{\textnormal{cal}}|$. Without loss of generality, we take $\Ical_{cal} = \{1,\dots,n\}$.
We allow the researcher to fit a predictive model on the training set $\Ical_{\textnormal{train}}$ using an arbitrary procedure, calling the result $\hat{f}$, a function from $\X$ to some space $\mathcal{Z}$.
The remainder of this paper shows how to subsequently create set-valued predictors from $\hat{f}$ that control a certain statistical error notion, regardless of the quality of the initial model fit or the distribution of the data.  For this task, we will only use the calibration points $(X_1,Y_1),\dots, (X_n,Y_n)$. See Figure~\ref{fig:data_split} for a visualization of our setting.

Next, let $\T : \X \to \Y'$ be a set-valued function (a \emph{tolerance region}) that maps a feature vector to a set-valued prediction. This function would typically be constructed from the predictive model, $\hat{f}$, which was fit on the training data---see the example in Figure~\ref{fig:mri_ex}. We will describe one possible construction in detail in Section~\ref{sec:making_sets}.
We further suppose we have a collection of such set-valued predictors indexed by a one-dimensional parameter $\lambda$ taking values in a closed set $\Lambda \subset \R \cup \{\pm\infty\}$ that are {\em nested}, meaning that larger values of $\lambda$ lead to larger sets:
\begin{equation}
\label{eq:nested_sets}
    \lambda_1 < \lambda_2 \implies \T_{\lambda_1}(x) \subset \T_{\lambda_2}(x).
\end{equation}

To capture a notion of error, let $L(y, \S) : \Y \times \Y' \to \R_{\ge 0}$ be a {\em loss function} on prediction sets. For example, we could take $L(y, \S) = \ind{y \notin \S}$, which is the loss function corresponding to classical tolerance regions. We require that the loss function respects the following nesting property:
\begin{equation}
\S \subset \S' \implies L(y, \S) \ge L(y, \S').
\label{eq:monotone_loss}
\end{equation}
That is, larger sets lead to smaller loss. 
We then define the {\em risk} of a set-valued predictor $\T$ to be
\begin{equation*}
R\big(\T) = \E\big[L(Y, \T(X))\big].
\end{equation*}
Since we will primarily be considering the risk of the tolerance functions from the family $\{\T_\lambda\}_{\lambda \in \Lambda}$, we will use the notational shorthand $R(\lambda)$ to mean $R(\T_\lambda)$. We further assume that there exists an element $\lambda_{\max}\in \Lambda$ such that $R(\lambda_{\max}) = 0$.

\subsection{The procedure}
Recalling Definition~\ref{def:rcps}, our goal is to find a set function whose risk is less than some user-specified threshold $\alpha$. 
To do this, we search across the collection of functions $\{\T_\lambda\}_{\lambda \in \Tau}$ and estimate their risk on data not used for model training, $\Ical_{cal}$.
We then show that by choosing the value of $\lambda$ in a certain way, we can guarantee that the procedure has risk less than $\alpha$ with high probability. 

We assume that we have access to a pointwise upper confidence bound (UCB) for the risk function for each $\lambda$:
\begin{equation}
    P\bigg(R(\lambda) \le \underbrace{\widehat{R}^+(\lambda)}_{\text{UCB}} \bigg) \ge 1 - \delta,
\label{eq:general_bound}
\end{equation}
where $\widehat{R}^+(\lambda)$ may depend on $(X_1, Y_1), \dots, (X_n, Y_n)$. 
We will present a generic strategy to obtain such bounds by inverting a concentration inequality as well as concrete bounds for various settings in Section~\ref{sec:concentration}. We choose $\lhat$ as the smallest value of $\lambda$ such that the entire confidence region to the right of $\lambda$ falls below the target risk level $\alpha$: 
\begin{equation}
\lhat \triangleq \inf \left\{\lambda \in \Lambda : \Rhat^+(\lambda') < \alpha, \,\, \forall \lambda' \ge \lambda \right\}.
\label{eq:lambda_hat_def}
\end{equation}
See Figure~\ref{fig:diffwidths} for a visualization.

\begin{figure}[ht]
    \centering
    \includegraphics[width=4in]{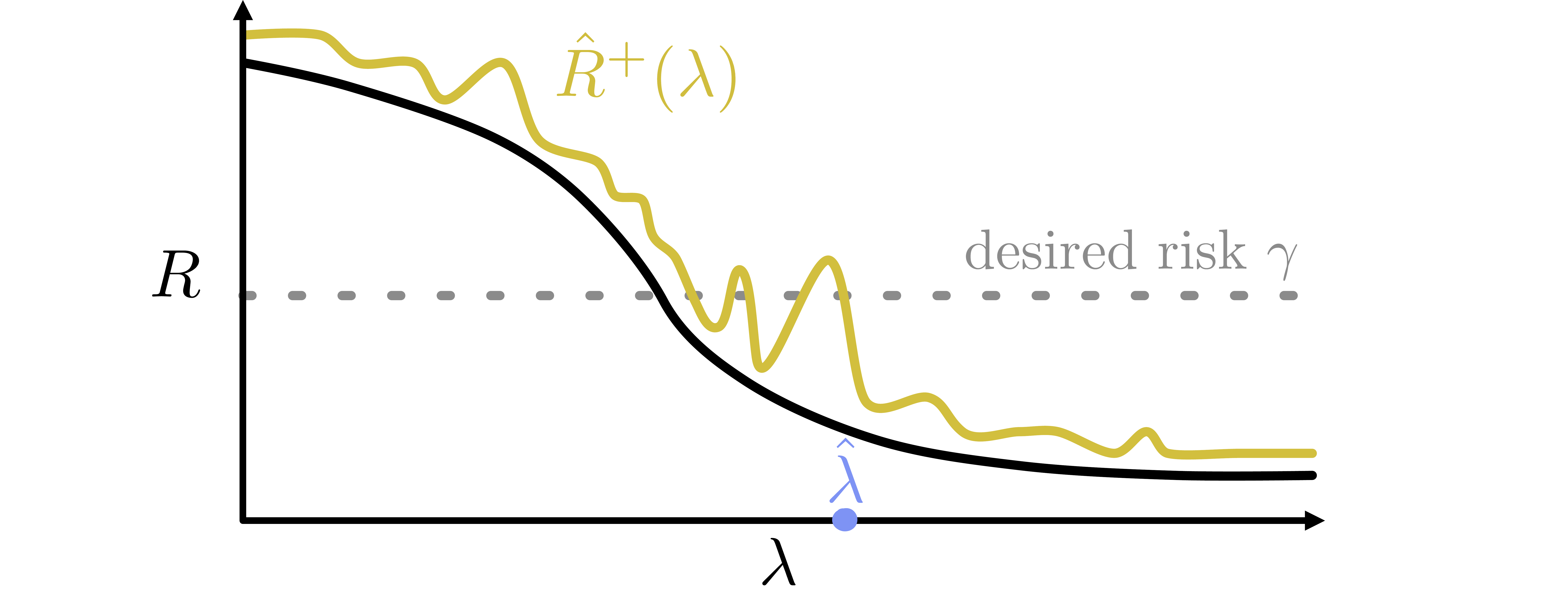}
    \caption{\textbf{Visualization of UCB calibration.}}
    \label{fig:diffwidths}
\end{figure}

This choice of $\lambda$ results in a set-valued predictor that controls the risk with high probability:
\begin{theorem}[Validity of UCB calibration]
    \label{thm:abstract_control}
    Let $(X_i, Y_i)_{i = 1,\dots,n}$ be an i.i.d.\  sample, let $L(\cdot, \cdot)$ be a loss satisfying the monotonicity condition in \eqref{eq:monotone_loss}, and let $\{\T_\lambda \}_{\lambda \in \Lambda}$ be a collection of set predictors satisfying the nesting property in \eqref{eq:nested_sets}.
    Suppose \eqref{eq:general_bound} holds pointwise for each $\lambda$, and that $R(\lambda)$ is continuous.
    Then for $\hat{\lambda}$ chosen as in \eqref{eq:lambda_hat_def},
    \begin{equation}
    P\left(R(\T_{\hat{\lambda}}) \le \alpha\right) \ge 1 - \delta.
    \end{equation}
    That is, $\T_{\hat{\lambda}}$ is a $(\alpha, \delta)$-RCPS.
\end{theorem}

All proofs are presented in Appendix~\ref{app:proofs}. 
Note that we are able to turn a pointwise convergence result into a result on the validity of a data-driven choice of $\lambda$.  This is due to the monotonicity of the risk function; without the monotonicity, we would need a uniform convergence result on the empirical risk in order to get a similar guarantee. Next, we will show how to get the required concentration in \eqref{eq:general_bound} for cases of interest, so that we can carry out the UCB calibration algorithm. Later, in Section~\ref{sec:experiments}, we will introduce several concrete loss functions and empirically evaluate the performance of the UCB calibration algorithms in a variety of prediction tasks.

\begin{remark}
Upper confidence bound calibration holds in more generality than the concrete instantiation above. The result holds for any monotone $R(\lambda)$ with a pointwise upper confidence bound $\hat{R}^+(\lambda)$. We present the general statement in Appendix~\ref{app:proofs}.
\end{remark}

\begin{remark}
The above result also implies that UCB calibration gives an RCPS even if the data used to fit the initial preditive model comes from a different distribution. The only requirement is that the calibration data and the test data come from the same distribution.
\end{remark}

\begin{remark}
We assumed that $R(\cdot)$ is continuous for simplicity, but this condition can be removed with minor modifications. The upper confidence bound is not assumed to be continuous.
\end{remark}

\section{Concentration Inequalities for the Upper Confidence Bound}
\label{sec:concentration}
In this section, we develop upper confidence bounds as in \eqref{eq:general_bound} under different conditions on the loss function, which will allow us to use the UCB calibration procedure for a variety of prediction tasks. In addition, for settings for which no finite-sample bound is available, we give an asymptotically valid upper confidence bound. Software implementing the upper confidence bounds is available in this project's \href{https://github.com/aangelopoulos/rcps}{\textcolor{blue}{public GitHub repository}} along with code to exactly reproduce our experimental results.

\subsection{Bounded losses}
We begin with the case where our loss is bounded above, and without loss of generality we take the bound to be one. We will present several upper confidence bounds and compare them in numerical experiments. The confidence bound of Waudby-Smith and Ramdas \cite{waudby2020variance} is the clear winner, and ultimately we recommend this bound for use in all cases with bounded loss.

\subsubsection{Illustrative case: the simplified Hoeffding bound}
\label{sec:hoeffding}
It is natural to construct an upper confidence bound for $R(\lambda)$ based on the empirical risk, the average loss of the set-valued predictor $\T_\lambda$ on the calibration set: 
\begin{equation*}
\widehat{R}(\lambda) \triangleq \frac{1}{n} \sum_{i=1}^n L\left(Y_i, \T_\lambda(X_i)\right).
\end{equation*}
As a warm-up, recall the following simple version of Hoeffding's inequality:
\begin{prop}[Hoeffding's inequality, simple version  \cite{hoeffding1963}]
Suppose the loss is bounded above by one. Then,
\begin{equation*}
P\left( \widehat{R}(\lambda) - R(\lambda) \le -x\right) \le \exp\{-2nx^2\}.
\label{eq:hoeffding_bound}
\end{equation*}
\end{prop}
This implies an upper confidence bound 
\begin{equation}\label{eq:Rhat_sim_hoeffding}
\widehat{R}^+_{{\rm sHoef}}(\lambda) = \widehat{R}(\lambda) + \sqrt{\frac{1}{2n}\log\left(\frac{1}{\delta}\right)}.
\end{equation}
Applying Theorem~\ref{thm:abstract_control} with 
\begin{align}
\lhat = \hat{\lambda}^{\textnormal{sHoef}} 
&\triangleq \inf \left\{\lambda \in \Lambda : \Rhat_{{\rm sHoef}}^+(\lambda') < \alpha, \,\, \forall \lambda' \ge \lambda \right\} \nonumber \\
&= \inf \left\{\lambda \in \Lambda : \widehat{R}(\lambda) < \alpha - \sqrt{\frac{1}{2n}\log\left(\frac{1}{\delta}\right)} \right\},
\label{eq:lhat_hoeffding}
\end{align}
we can generate an RCPS, which we record formally below.
\begin{theorem}[RCPS from Hoeffding's inequality]
\label{thm:hoeff}
In the setting of Theorem~\ref{thm:abstract_control}, assume additionally that the loss is bounded by one. Then, $\T_{\hat{\lambda}^{\textnormal{sHoef}}}$ is a $(\alpha, \delta)$-RCPS.
\end{theorem}
In view of \eqref{eq:lhat_hoeffding}, UCB calibration with this version of Hoeffding's bound results in a procedure that is simple to state---one selects the smallest set size such that the empirical risk on the calibration set is below $\alpha - \sqrt{\log(1/\delta) / 2n}$. This result is only presented for illustration purposes, however. Much tighter  concentration results are available, so in practice we recommend using the better bounds described next.

\subsubsection{Hoeffding--Bentkus bound}
\label{sec:hbb}
In general, an upper confidence bound can be obtained if the lower tail probability of $\widehat{R}(\lambda)$ can be controlled, in the following sense:
\begin{prop}\label{prop:generic}
Suppose $g(t; R)$ is a nondecreasing function in $t\in \mathbb{R}$ for every $R$:
\begin{equation}\label{eq:generic}
P(\widehat{R}(\lambda)\le t)\le g(t; R(\lambda)).
\end{equation}
Then, $\widehat{R}^{+}(\lambda) = \sup\left\{R: g(\widehat{R}(\lambda); R)\ge \delta\right\}$ satisfies \eqref{eq:general_bound}.
\end{prop}
This result shows how a tail probability bound can be inverted to yield an upper confidence bound. Put another way, $g(\widehat{R}(\lambda); R)$ is a conservative p-value for testing the one-sided null hypothesis $H_0: R(\lambda)\ge R$. From this perspective, Proposition \ref{prop:generic} is simply a restatement of the duality between p-values and confidence intervals.

The previous discussion of the simple Hoeffding bound is a special case of this proposition, but stronger results are possible. The rest of this section develops a sharper tail bound that builds on two stronger concentration inequalities. 

We begin with a tighter version of Hoeffding's inequality. 
\begin{prop}[Hoeffding's inequality, tighter version \cite{hoeffding1963}]\label{prop:hoeffding}
Suppose the loss is bounded above by one. Then for any $t < R(\lambda)$,
\begin{equation}
P\left(\widehat{R}(\lambda)\le t\right) \le \exp\{-nh_1(t; R(\lambda))\},
\label{eq:hoeffding_bound_tight}
\end{equation}
where $h_1(t; R) = t \log(t / R) + (1 - t) \log((1-t) / (1- R))$.
\end{prop}
The weaker Hoeffding inequality is implied by Proposition \ref{prop:hoeffding} using the fact that $h_1(t; R)\ge 2(t - R)^2$. Another strong inequality is the Bentkus inequality, which implies that the Binomial distribution is the worst case up to a small constant.
The Bentkus inequality is nearly tight if the loss function is binary, in which case $n\Rhat(\lambda)$ is binomial. 

\begin{prop}[Bentkus inequality \cite{bentkus2004hoeffding}]\label{prop:bentkus}
Suppose the loss is bounded above by one. Then,
\begin{equation}
P\left( \widehat{R}(\lambda) \le t\right) \le eP\left({\rm Binom}(n,R(\lambda)) \le \lceil nt \rceil\right),
\label{eq:bentkus_bound}
\end{equation}
where ${\rm Binom}(n,p)$ denotes a binomial random variable with sample size $n$ and success probability $p$.
\end{prop}

Putting Proposition \ref{prop:hoeffding} and  \ref{prop:bentkus} together, we obtain a lower tail probability bound for $\Rhat(\lambda)$:
\begin{align}
g^{\mathrm{HB}}(t; R(\lambda)) \triangleq \min\Bigg(& \exp\{-nh_1(t; R(\lambda))\}, \,\, eP\left({\rm Binom}(n,R(\lambda)) \le \lceil nt \rceil\right)\Bigg).\label{eq:HB_g}
\end{align}
By Proposition \ref{prop:generic}, we obtain a $(1 - \delta)$ upper confidence bound for $R(\lambda)$ as 
\begin{equation}\label{eq:Rhat_HB}
\Rhat_{\mathrm{HB}}^+(\lambda) = \sup\Big\{R: g^{\mathrm{HB}}(\Rhat(\lambda); R)\ge \delta\Big\}.
\end{equation}
We obtain $\lhat^{{\rm HB}}$ from $\Rhat_{\mathrm{HB}}^+(\lambda)$ as in \eqref{eq:lambda_hat_def} and conclude the following:
\begin{theorem}[RCPS from the Hoeffding--Bentkus bound]\label{thm:HB}
    In the setting of Theorem~\ref{thm:abstract_control}, assume additionally that the loss is bounded by one. Then, $\T_{\lhat^{{\rm HB}}}$ is a $(\alpha,\delta)$-RCPS.
\end{theorem}

\begin{remark}
 The Bentkus inequality is closely related to an exact confidence region for the mean of a binomial distribution. In the special where the loss takes values only in $\{0,1\}$, this exact binomial result gives the most precise upper confidence bound and should always be used; see Appendix~\ref{app:binary_loss}.
\end{remark}

\subsubsection{Waudby-Smith--Ramdas bound}
Although the Hoeffding--Bentkus bound is nearly tight for binary loss functions, for non-binary loss functions, it can be very loose because it does not adapt to the variance of $L(Y_i, \T_\lambda(X_i))$. As an example, consider the extreme case where $\mathrm{Var}(L(Y_i, \T_\lambda(X_i))) = 0$, then $\widehat{R}(\lambda) = R(\lambda)$ almost surely, and hence $\widehat{R}^+(\lambda)$ can be set as $\widehat{R}(\lambda)$. In general, when $\mathrm{Var}(L(Y_i, \T_\lambda(X_i)))$ is small, the tail probability bound can be much tighter than that given by the Hoeffding--Bentkus bound. We next present a bound that is adaptive to the variance and improves upon the previous result in most settings.

The most well-known concentration result incorporating the variance is Bernstein's inequality \citep{bernstein1924modification}. To accommodate the case where the variance is unknown and must be estimated, \cite{maurer2009empirical} proposed an empirical Bernstein inequality which replaces the variance by the empirical variance estimate. This implies the following upper confidence bound for $R(\lambda)$:
\begin{equation}\label{eq:eBern}
\Rhat_{\mathrm{eBern}}^+(\lambda) = \Rhat(\lambda) +  \hat{\sigma}(\lambda)\sqrt{\frac{2\log (2 / \delta)}{n}} + \frac{7\log (2 / \delta)}{3(n - 1)}, \quad \text{where }\hat{\sigma}^2(\lambda) = \frac{1}{n-1}\sum_{i=1}^{n}(L(Y_i, T_\lambda(X_i)) - \Rhat(\lambda))^2.
\end{equation}
However, the constants in the empirical Bernstein inequality are not tight, and improvements are possible.

As an alternative bound that adapts to the unknown variance, \cite{waudby2020variance} recently proposed the \emph{hedged capital confidence interval} for the mean of bounded random variables, drastically improving upon the empirical Bernstein inequality. Unlike all aforementioned bounds, it is not based on inverting a tail probability bound for $\Rhat(\lambda)$, but instead builds on tools from online inference and martingale analysis. 
For our purposes, we consider an one-sided variant of their result, which we refer to as the Waudby-Smith--Ramdas (WSR) bound. 

\begin{prop}[Waudby-Smith--Ramdas bound \cite{waudby2020variance}]\label{prop:WSR}
Let $L_i(\lambda) = L(Y_i, T_\lambda(X_i))$ and 
$$
\hat{\mu}_i(\lambda) = \frac{1/2 + \sum_{j=1}^{i}L_{j}(\lambda)}{1 + i}, \,\, \hat{\sigma}_i^2(\lambda) = \frac{1/4 + \sum_{j=1}^{i}(L_{j}(\lambda) - \hat{\mu}_{j}(\lambda))^2}{1 + i}, \,\, \nu_{i}(\lambda) = \min\left\{1, \sqrt{\frac{2\log(1 / \delta)}{n\hat{\sigma}_{i-1}^2(\lambda)}}\right\}.
$$
Further let 
$$
\mathcal{K}_i(R; \lambda) = \prod_{j=1}^{i}\left\{1 - \nu_j(\lambda) (L_j(\lambda) - R)\right\}, \quad \Rhat_{\mathrm{WSR}}^+(\lambda) = \inf\left\{R \ge 0: \max_{i = 1,\ldots, n}\mathcal{K}_i(R; \lambda) > \frac{1}{\delta}\right\}.
$$
Then $\Rhat_{\mathrm{WSR}}^+(\lambda)$ is a $(1 - \delta)$ upper confidence bound for $R(\lambda)$.
\end{prop}
Since the result is a small modification of the one stated in \cite{waudby2020variance}, for completeness we present a proof in Appendix \ref{app:proofs}.
As before, we then set $\lhat^{{\rm WSR}}$ as in \eqref{eq:lambda_hat_def} to obtain the following corollary:
\begin{theorem}[RCPS from the Waudby-Smith--Ramdas bound]\label{thm:WSR}
    In the setting of Theorem~\ref{thm:abstract_control}, assume additionally that the loss is bounded by $1$. Then, $\T_{\lhat^{{\rm WSR}}}$ is a $(\alpha,\delta)$-RCPS.
\end{theorem}

\subsubsection{Numerical experiments for bounded losses}
We now evaluate the aforementioned bounds on random samples from a variety of distributions on $[0,1]$. As an additional point of comparison, we also consider a bound based on the central limit theorem (CLT) that does not have finite-sample guarantees, formally defined later in Section~\ref{sec:clt_control}. In particular, given a distribution $F$ for the loss $L(Y, \T_\lambda(X))$, we sample $L_1, \ldots, L_n\stackrel{\mathrm{i.i.d.}}{\sim} F$ and compute the $(1 - \delta)$ upper confidence bound of the mean for $n\in \{\lfloor 10^r\rfloor: r = 2, 2.5, 3, 3.5, 4\}$ and $\delta \in \{0.1, 0.01, 0.001\}$.  We present the results for $\delta = 0.1$ here and report on other choices of $\delta$s in Appendix \ref{app:additional_plots}. 
Based on one million replicates of each setting, we report the coverage and the median gap between the UCB and true mean; the former measures the validity and the latter measures the power.

We consider the Bernoulli distribution, $F = \mathrm{Ber}(\mu)$, and the Beta distribution, $F = \mathrm{Beta}(a, b)$ with $b = a(1/\mu - 1)$. Note that both distributions have mean $\mu$. Since a user would generally be interested in setting $\alpha$ in $[0.001, 0.1]$ in practice, we set $\mu \in \{0.1, 0.01, 0.001\}$. To account for different levels of variability, we set $a\in \{0.1, 1, 10\}$ for the Beta distribution, with a larger $a$ yielding a tighter concentration around the mean. 
We summarize the results in Figure \ref{subfig:UCB_bounded_coverage_delta0.1}. First, we observe that the CLT does not always have correct coverage, especially when the true mean is small, unless the sample size is large. Accordingly, we recommend the bounds with finite-sample guarantees in this case. Next, as shown in Figure \ref{subfig:UCB_bounded_medgap_delta0.1}, the WSR bound outperforms the others for all Beta distributions and has a similar performance to the HB bound for Bernoulli distributions. It is not surprising that the HB bound performs well for binary variables since the Bentkus inequality is nearly tight here. 
Based on these observations, we recommend the WSR bound for any non-binary bounded loss. When the loss is binary, one should use the exact result based on quantiles of the binomial distribution; see Appendix~\ref{app:binary_loss}.

\begin{figure}
\begin{subfigure}{0.49\textwidth}
\includegraphics[width = 0.98\textwidth]{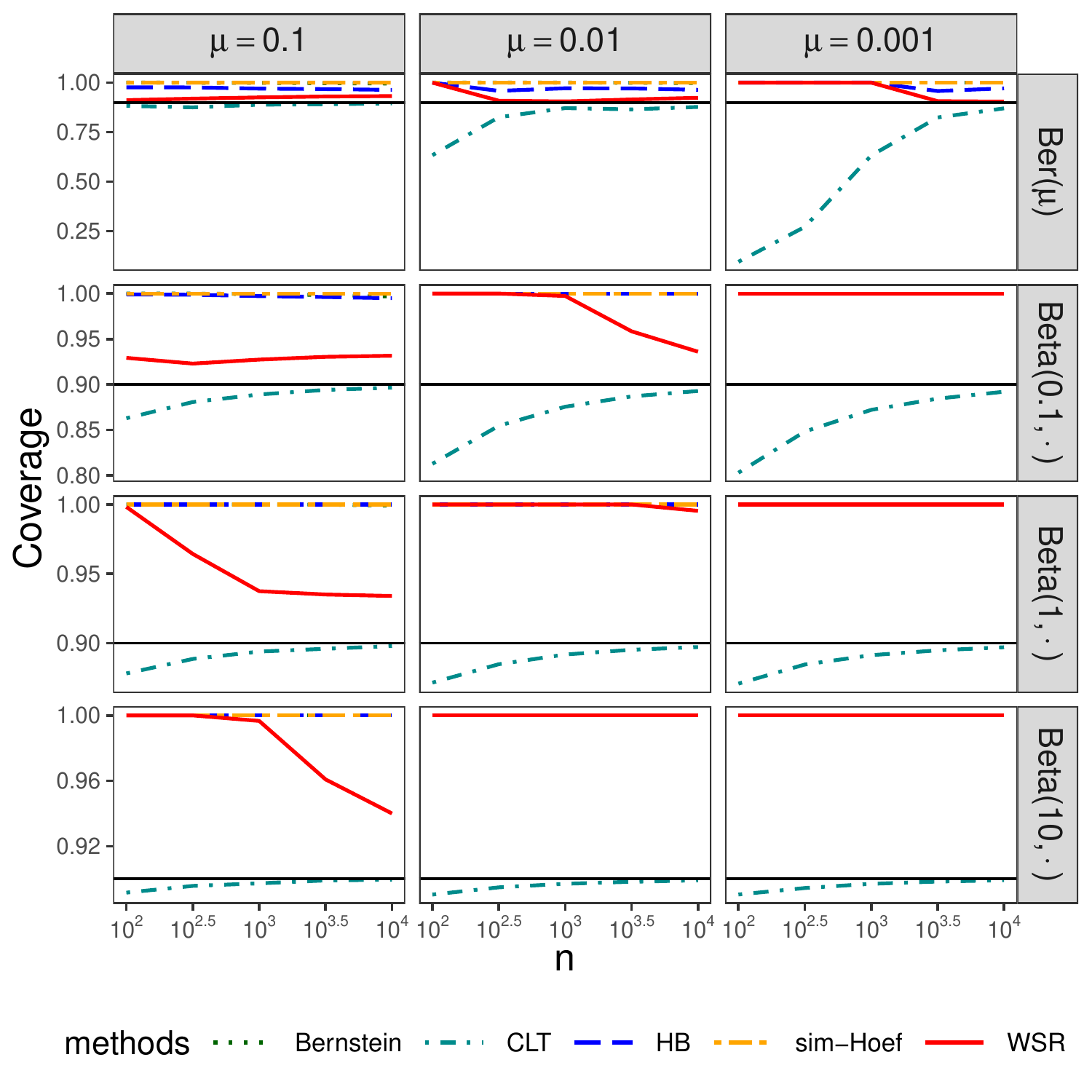}
\caption{Coverage $P(\Rhat(\lambda) \ge R(\lambda))$}\label{subfig:UCB_bounded_coverage_delta0.1}
\end{subfigure}
\begin{subfigure}{0.49\textwidth}
\includegraphics[width = 0.98\textwidth]{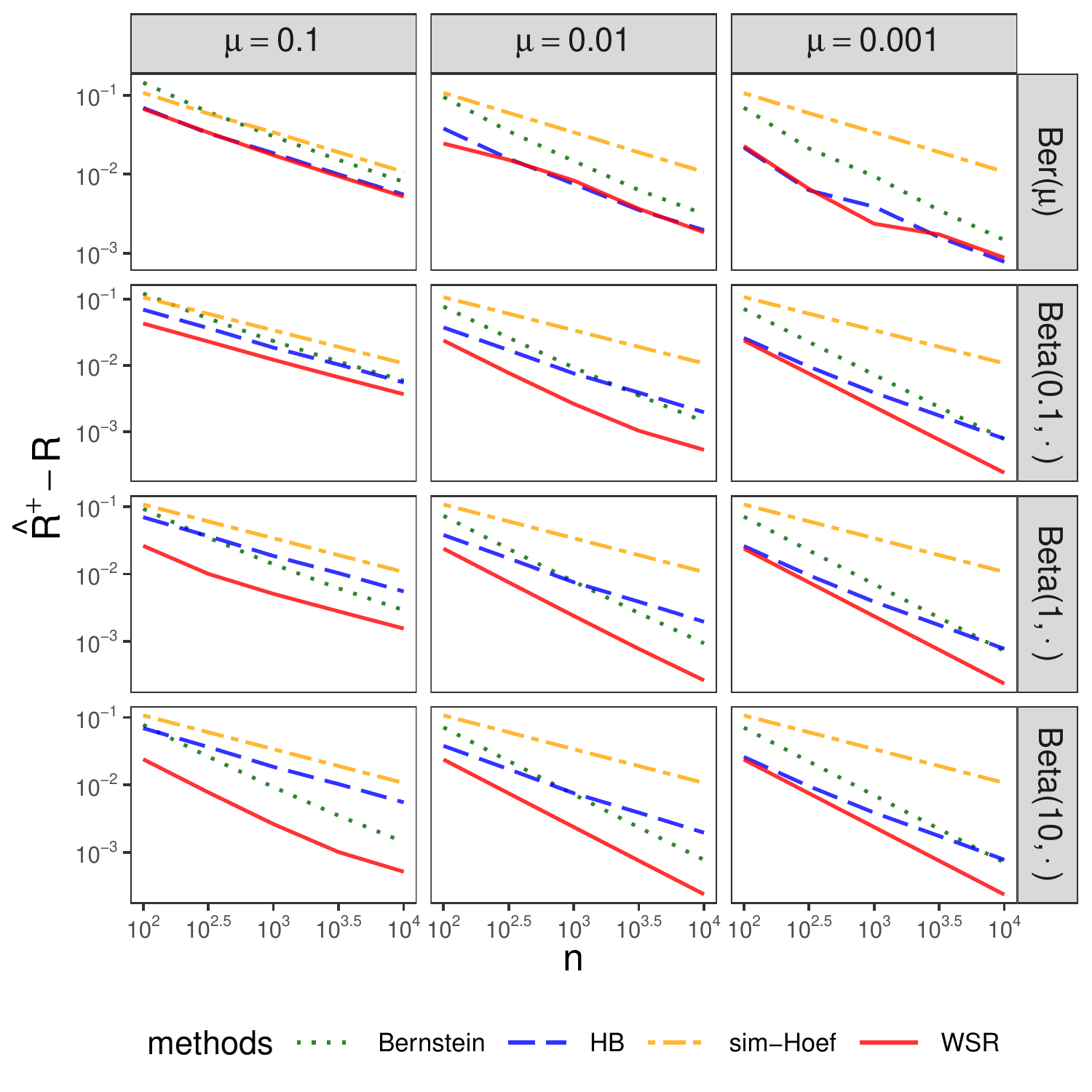}
\caption{Median of $\Rhat^+(\lambda) - R(\lambda)$}\label{subfig:UCB_bounded_medgap_delta0.1}
\end{subfigure}
\caption{\textbf{Numerical evaluations of concentration results for bounded losses}. We show the simple Hoeffding bound \eqref{eq:Rhat_sim_hoeffding}, HB bound \eqref{eq:Rhat_HB}, empirical Bernstein bound \eqref{eq:eBern}, CLT bound \eqref{eq:Rhat_CLT}, and WSR bound (Proposition \ref{prop:WSR}) with sample size $n$. Each row corresponds to a type of distribution and each column corresponds to a value of the mean. The CLT bound is excluded in (b) because it does not achieve the target coverage in most of the cases.}\label{fig:UCB_bounded_delta0.1}
\end{figure}

\subsection{Unbounded losses}
We now consider the more challenging case of unbounded losses. As a motivating example, consider the Euclidean distance of a point to its closest point in the prediction set as a loss:
\begin{equation*}
    L\left(y, \S \right) = {\rm inf}\Big\{||y-y'||_2 \; : \; y' \in \S) \Big\}.
\end{equation*}
Based on the well-known results of \cite{bahadur1956}, we can show that it is impossible to derive a nontrivial upper confidence bound for the mean of nonnegative random variables in finite samples without any other restrictions---see Proposition \ref{prop:impossibility} in Appendix \ref{app:proofs}. As a result, we must restrict our attention to distributions that satisfy some regularity conditions. One reasonable approach is to consider distributions satisfying a bound on the coefficient of variation, and we turn our attention to such distributions next.

\subsubsection{The Pinelis--Utev inequality}
For nonnegative random variables with bounded coeffecient of variation, the Pinelis--Utev inequality gives a tail bound as follows:
\begin{prop}[Pinelis and Utev \cite{pinelis1989}, Theorem 7]
\label{prop:pinelis_utev}
    Let 
    $c_v(\lambda) = \sigma(\lambda) / R(\lambda)$
    denote the coefficient of variation. Then for any $t \in (0, R(\lambda)]$,
    \begin{equation}
    P(\Rhat(\lambda) \le t) 
    \le \exp \left\{-\frac{n}{c_v^2(\lambda) + 1}\left[1 + \frac{t}{R(\lambda)}\log\left(\frac{t}{e R(\lambda)}\right) \right] \right\}.
\end{equation}
\end{prop}
By Proposition \ref{prop:generic}, this implies an upper confidence bound of $R(\lambda)$: 
\begin{equation}\label{eq:Rhat_PU}
\Rhat_{\mathrm{PU}}^+(\lambda) = \sup\left\{R: \exp \left\{-\frac{n}{c_v^2(\lambda) + 1}\left[1 + \frac{\Rhat(\lambda)}{R}\log\left(\frac{\Rhat(\lambda)}{e R}\right) \right] \right\}\ge \delta\right\}.
\end{equation}
This result shows that a nontrivial upper confidence bound can be derived if $c_v(\lambda)$ is known. When $c_v(\lambda)$ is unknown, we can treat it as a sensitivity parameter or estimate it based on the sample moments.   
Using this inequality and plugging in an upper bound $c_v$ for $c_v(\lambda)$, we define $\hat{\lambda}^{\textnormal{PU}}$ with the UCB calibration procedure (i.e, as in \eqref{eq:lambda_hat_def}) to get the following guarantee:
\begin{theorem}[RCPS from Pinelis--Utev inequality]
\label{thm:PU}
In the setting of Theorem~\ref{thm:abstract_control}, suppose in addition that for each $\lambda \in \Lambda$, $c_v(\lambda) \le c_v$ for some constant $c_v$. Then, $\T_{\hat{\lambda}^{\textnormal{PU}}}$ is a $(\alpha, \delta)$-RCPS.
\end{theorem}

\subsubsection{Numerical comparisons of upper confidence bounds}

Next, we numerically study the unbounded case with two competing bounds---the PU bound with $c_v$ estimated by the ratio between the standard error and the average, and a bound based on the CLT described explicitly later in Section~\ref{sec:clt_control} (which does not have finite-sample coverage guarantees). We consider three types of distributions---the Gamma distribution $\Gamma(a, 1)$, the square-t distribution $t^2(v)$ (the distribution of the square of a $t$-distributed variable with degree of freedom $v$), and the log-normal distribution $\mathrm{LN}(\mu, \sigma)$ (the distribution of $\exp(Z)$ where $Z \sim N(\mu, \sigma)$). For each distribution, we consider a light-tailed and a heavy-tailed setting, and normalize the distributions to have mean $\mu = 1$. The parameter settings are summarized in Table \ref{tab:unbounded}. 

Conducting our experiments as in the bounded case, we present the coverage and median gap with $\delta = 0.1$ in Figure \ref{fig:UCB_unbounded_delta0.1}. From Figure \ref{subfig:UCB_unbounded_coverage_delta0.1}, we see that the CLT bound nearly achieves the desired coverage for light-tailed distributions but  drastically undercovers for heavy-tailed distributions. By contrast, the PU bound has valid coverage in these settings. From Figure \ref{subfig:UCB_unbounded_medgap_delta0.1}, we see that the CLT bound is much tighter in all cases, though the gap bewteen two bounds shrink as the sample size grows. Therefore, we recommend the CLT bound when the losses are believed to be light-tailed and the sample size is moderately large, and the PU bound otherwise. 

\begin{table}[H]
    \centering
    \begin{tabular}{c|ccc}
        \toprule
         & Gamma & Squared-t & Log-normal \\
        \midrule
        light-tailed & $a = 1$ & $v = 100$ & $(\mu, \sigma) = (-0.125, 0.5)$ \\
        heavy-tailed & $a = 0.05$ & $v = 4$ & $(\mu, \sigma) = (-2, 2)$ \\
        \bottomrule
    \end{tabular}
    \caption{Distributions considered for the unbounded case.}
    \label{tab:unbounded}
\end{table}

\begin{figure}
\begin{subfigure}{0.49\textwidth}
\includegraphics[width = 0.98\textwidth]{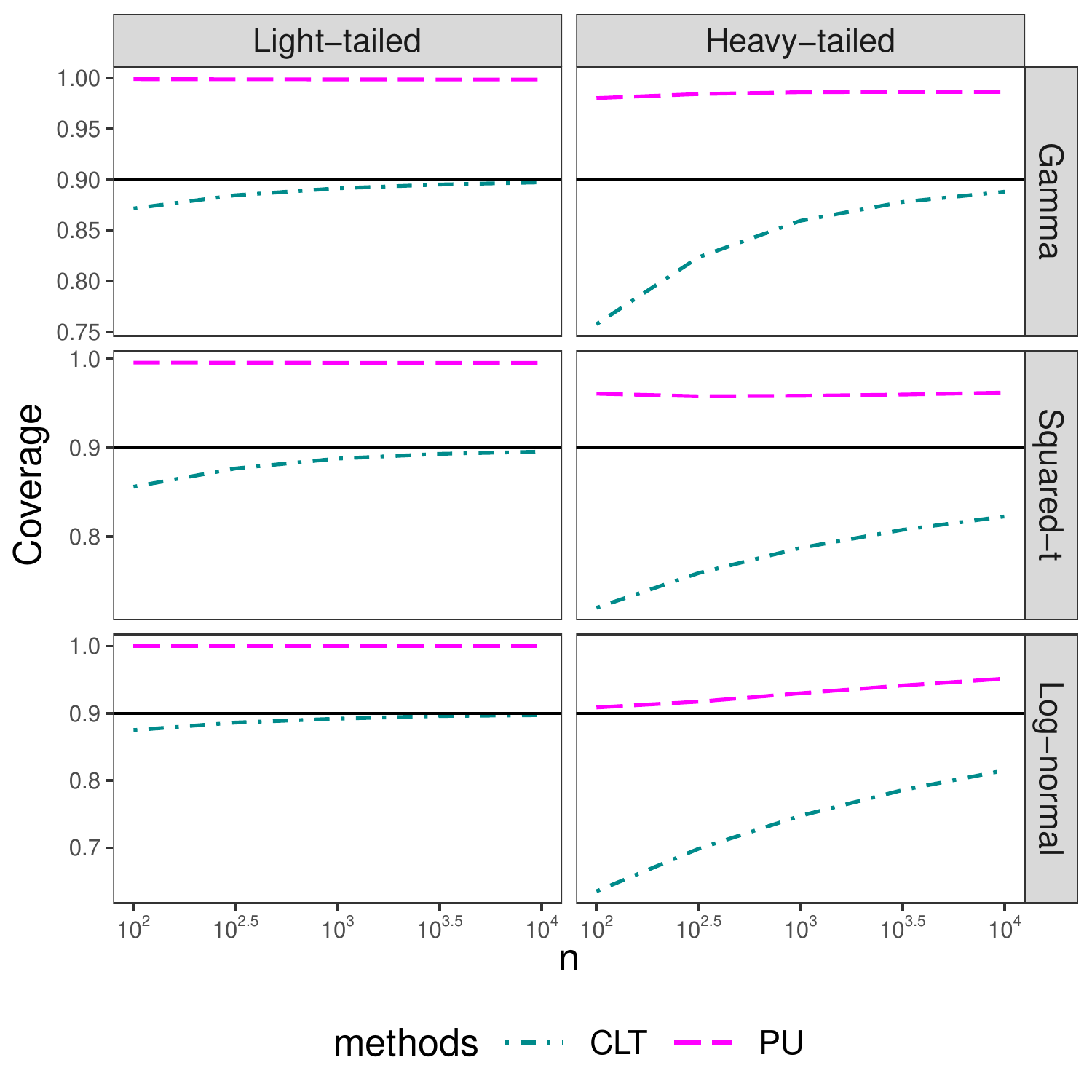}
\caption{Coverage $P(\Rhat(\lambda) \ge R(\lambda))$}\label{subfig:UCB_unbounded_coverage_delta0.1}
\end{subfigure}
\begin{subfigure}{0.49\textwidth}
\includegraphics[width = 0.98\textwidth]{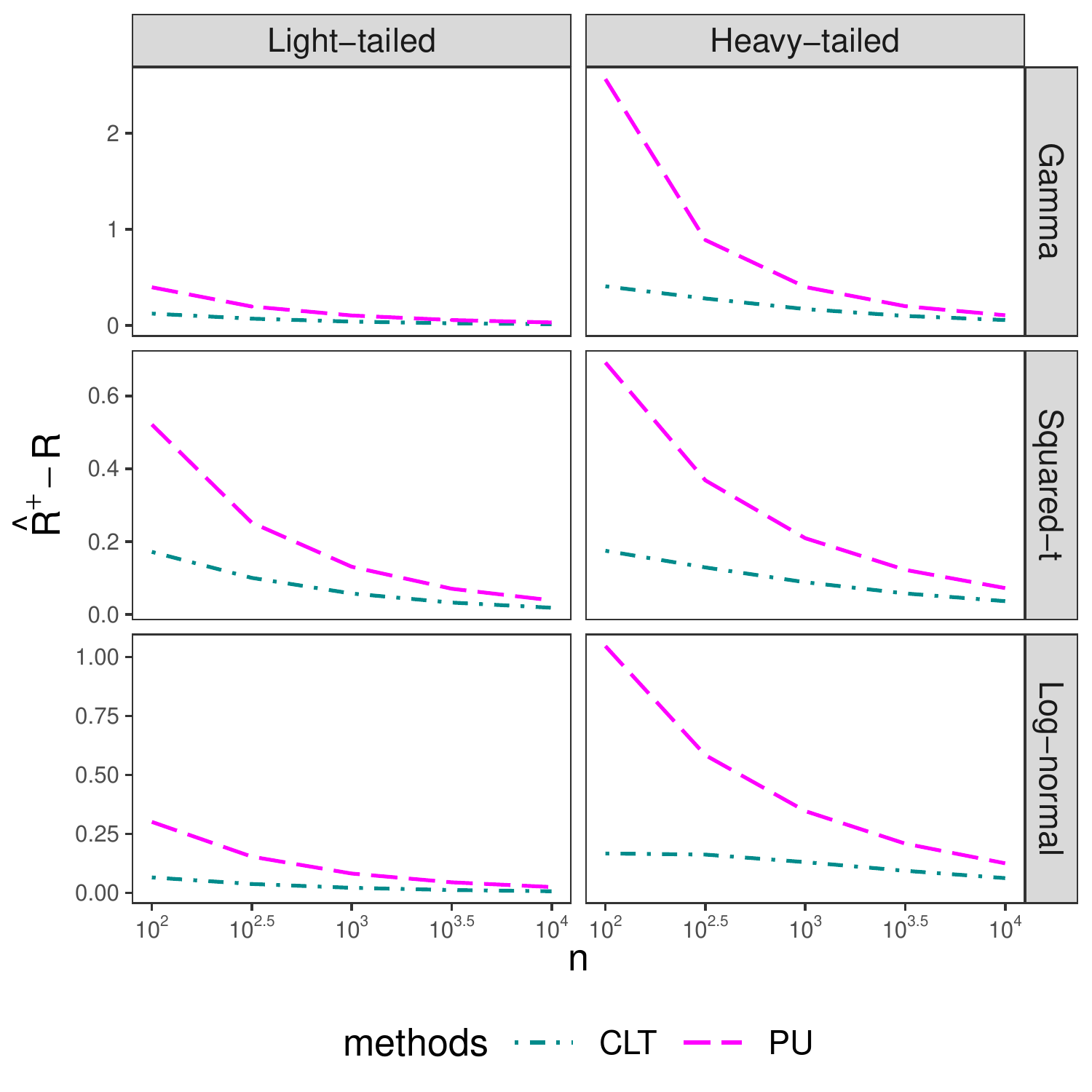}
\caption{Median of $\Rhat^+(\lambda) - R(\lambda)$}\label{subfig:UCB_unbounded_medgap_delta0.1}
\end{subfigure}
\caption{\textbf{Numerical evaluations of the PU bound.} We compare the bound from \eqref{eq:Rhat_PU} with the estimated coefficient of variation and the CLT bound \eqref{eq:Rhat_CLT}, with sample size $n$ from each distribution in Table \ref{tab:unbounded}. Each row corresponds to a type of distribution and each column corresponds to a value of the mean. }\label{fig:UCB_unbounded_delta0.1}
\end{figure}

\subsection{Asymptotic results}
\label{sec:clt_control}

When no finite-sample result is available, we can still use the UCB calibration procedure to get prediction sets with asymptotic validity. Suppose the loss $L(Y, \T_\lambda(X))$ has a finite second moment for each $\lambda$. Then, since the losses for each $\lambda$ are i.i.d., we can apply the CLT to get
\begin{equation}
   \underset{n \to \infty}{\rm lim} P\left(  \frac{\sqrt{n} (\widehat{R}(\lambda) - R(\lambda))}{\hat{\sigma}(\lambda)} \le -t \right) \le \Phi\left(-t \right),
\label{eq:clt_upperbound}
\end{equation}
where $\Phi$ denotes the standard normal cumulative distribution function (CDF). This yields an asymptotic upper confidence bound for $R(\lambda)$: 
\begin{equation}\label{eq:Rhat_CLT}
\Rhat_{{\rm CLT}}^+(\lambda) = \widehat{R}(\lambda) + \frac{\Phi^{-1}(1-\delta)\hat{\sigma}(\lambda)}{\sqrt{n}}.
\end{equation}
Let $\hat{\lambda}^{\textnormal{CLT}} = \inf \{\lambda \in \Lambda : \Rhat_{{\rm CLT}}^+(\lambda') < \alpha, \,\, \forall \lambda' \ge \lambda \}$. Then, $\T_{\lhat^{{\rm CLT}}}$ is an asymptotic RCPS, as stated next.
\begin{theorem}[Asymptotically valid RCPS]\label{thm:clt_coverage}
In the setting of Theorem~\ref{thm:abstract_control}, assume additionally that $L(Y, \T_\lambda(X))$ has a finite second moment for each $\lambda$. Then,
\begin{equation}
\limsup_{n \to \infty} P\left(R(\T_{\hat{\lambda}^{\textnormal{CLT}}}) > \alpha \right) \le \delta.
\end{equation}
\end{theorem}
As a technical remark, note this result requires only a pointwise CLT for each $\lambda \in \Lambda$, analogously to the finite-sample version presented in Theorem~\ref{thm:abstract_control}.
Since this asymptotic guarantee holds for many realistic choices of loss function and data-generating distribution, this approximate version of UCB calibration greatly extends the reach of our proposed method.

\subsection{How large should the calibration set be?}
The numerical results presented previously give rough guidance as to the required size of the calibration set. While UCB calibration is always guaranteed to control the risk by Theorem~\ref{thm:abstract_control}, if the calibration set is too small then the sets may be larger than necessary.
Since our procedure finds the last point where the UCB $\widehat{R}^+(\lambda)$ is above the desired level $\alpha$, it will produce sets that are nearly as small as possible when $\widehat{R}^+(\lambda)$ is close to the true risk $R(\lambda)$. As a rule of thumb, we say that we have a sufficient number of calibration points when $\widehat{R}^+(\lambda)$ is within $10\%$ of $R(\lambda)$. The sample size required will vary with the problem setting, but use this heuristic to analyze our simulation results to get a few representative values.

Figure~\ref{subfig:UCB_bounded_medgap_delta0.1} reports on the bounded loss case. The left column shows that when we seek to control the risk at the relatively loose $\alpha = 0.1$ level, around $1,000$ calibration points suffice; the middle panel shows that when we seek to control the risk at level $\alpha = 0.01$, a few thousand calibration points suffice; and the right column shows that for the strict risk level $\alpha = 0.001$, about $10,000$ calibration points suffice. The required number of samples will increase slightly if we ask for a higher confidence level (i.e., smaller $\delta$), but the dependence on $\delta$ is minimal since the bounds will roughly scale as $\log(1/\delta)$---this scaling can be seen explicitly in the simple Hoeffding bound \eqref{eq:Rhat_sim_hoeffding}. Examining the unbounded loss examples presented in Figure~\ref{subfig:UCB_unbounded_medgap_delta0.1}, we see that about $1,000$ calibration points suffice for the student-t and log-normal examples, but that about $10,000$ calibration points are needed for the Gamma example.
In summary, $1,000$ to $10,000$ calibration points are sufficient to generate prediction sets that are not too conservative, i.e., sets that have risk that are not far below the desired level $\alpha$.

\section{Generating the Set-Valued Predictors}
\label{sec:making_sets}

In this section, we describe one possible construction of the nested prediction sets $\T_{\lambda}(x)$ from a given predictor $\hat{f}$. 
Any collection of the sets can be used to control the risk by Theorem~\ref{thm:abstract_control}, but some may produce larger sets than others. Here, we present one choice and show that it is approximately optimal for an important class of losses.

In the following subsections, we denote the infinitesimal risk of a continuous response $y$ with respect to a set $\S\subseteq\Y$ as its {\em conditional risk density},
\begin{equation}
    \rho_{x}(y,\S)=L(y,\S)p_{Y|X=x}(y).
\end{equation}
We will present the results for the case where $y$ is continuous, but the same algorithm and theoretical result hold in the discrete case if we instead take $\rho_x(y,\S)=L(y,\S)P(Y=y|X=x)$.

\subsection{A greedy procedure}
\label{sec:greedy-algorithm}
We now describe a construction of the tolerance functions $\T_\lambda$ based on the estimated conditional risk density.
We assume that our predictor is $\hat{p}_x(y)$, an estimate of $p_{Y|X=x}(y)$, and we let $\hat{\rho}_x(y,\S) = L(y,\S)\hat{p}_x(y)$.
Algorithm~\ref{alg:greedy} indexes a family of sets $\T_{\lambda}$ nested in $\lambda\le 0$ by iteratively including the riskiest portions of $\Y$, then re-computing the risk densities of the remaining elements.
The general greedy procedure is computationally convenient; moreover, it is approximately optimal for a large class of useful loss functions, as we will prove soon.

\begin{algorithm}[t]
    \caption{Greedy Sets}
    \label{alg:greedy}
    \begin{algorithmic}[1] 
        \Require $\lambda$, risk density estimate $\hat{\rho}_x$, step size $d\zeta$ 
        \Procedure{GreedySets}{$\lambda, \hat{\rho}_x$}
        \State $\T \gets \emptyset$
        \State $\zeta \gets $ a large number (e.g., $B$ in the bounded case)
        \While{$\zeta > -\lambda$}
            \vspace{1mm}
            \State $\zeta \gets \zeta - d\zeta$
            \vspace{1mm}
            \State $\T \gets \T \cup \big\{ y' \in \T^c : \hat{\rho}_{x}(y',\T) > \zeta \big\}$
        \EndWhile
        \State \textbf{return} $\T$
        \EndProcedure
        \Ensure The nested set with parameter $\lambda$ at $x$: $\T_\lambda(x)$
    \end{algorithmic}
\end{algorithm}

\begin{remark}
    Algorithm~\ref{alg:greedy} is greedy because it only considers the next $d\zeta$ portion of risk to choose which element to add to the current set. One can imagine versions of this algorithm which look ahead several steps. Such schemes may be tractable in some cases, but are generally much more computationally expensive.
\end{remark}
\subsection{Optimality properties of the greedy procedure}
\label{sec:unconditional-oracle}
Next, we outline a setting where our greedy algorithm is optimal.
Suppose our loss function has the simple form $L(y,\S)=L_y\ind{y\not\in\S}$, for constants $L_y$. 
 This assumption on $L$ describes the case where every $y$ has a different, fixed loss if it is not present in the prediction set, such as in our MRI classification example in the introduction.
In this case, the sets returned by Algorithm~\ref{alg:greedy} have the form
\begin{equation}
\label{eq:greedy-family}
    \T_\lambda(x) = \big\{y' : \hat{\rho}_x(y', \emptyset) \geq \zeta(\lambda) \big\}.
\end{equation}
That is, we return the set of response variables with risk density above some threshold; see Figure~\ref{fig:riskscore} for a visualization.

\begin{figure}[t]
    \centering
    \includegraphics[width=4in]{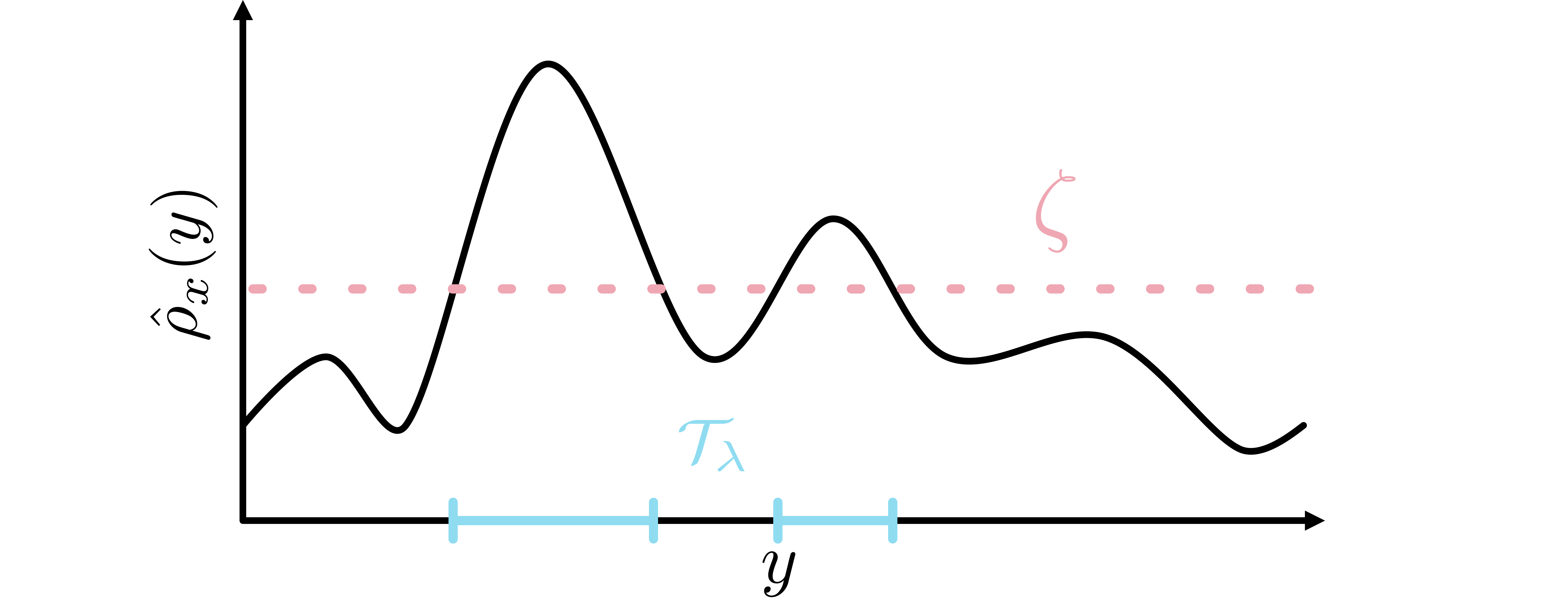}
    \vspace{-0.5cm}
    \caption{\textbf{Optimal prediction sets.}  In the special case where $\hat{\rho}_x(y, \S)$ does not depend on $\S$, $\T_\lambda(x)$ from Algorithm~\ref{alg:greedy} is made up of the $y \in \mathcal{Y}$ whose conditional risk density exceeds a threshold $\zeta$.}
    \label{fig:riskscore}
\end{figure}

Now, imagine that we know the exact conditional probability density, $p_{Y|X=x}(y)$, and therefore the exact $\rho_x(y,\S)$.
The prediction sets produced by Algorithm~\ref{alg:greedy} then have the smallest average size among all procedure that control the risk, as stated next.
\begin{theorem}[Optimality of the greedy sets]
In the setting above, let $\T':\mathcal{X} \to \Y'$ be any set-valued predictor such that $R(\T') \leq R(\T_\lambda)$, where $\T_\lambda$ is given by Algorithm~\ref{alg:greedy}.
Then,
\begin{equation}
\E[|\T_{\lambda}(X)|] \leq \E[|\T'(X)|].
\end{equation}
\label{thm:unconditional-optimality}
\vspace{-0.5cm}
\end{theorem}
Here, $|\cdot|$ denotes the set size: Lebesgue measure for continuous variables and counting measure for discrete variables. 
This result is a generalization of a result of \cite{Sadinle2016LeastAS} to our risk-control setting.
While we do not exactly know the risk density in practice and must instead use a plug-in estimate, this result gives us confidence that our set construction is a sensible one.
The choice of the parameterization of the nested sets is the analogue to the choice of the score function in the more specialized setting of conformal prediction \cite{gupta2020nested}, and it is known in that case that there are many choices that each have their own advantages and disadvantages. See \cite{Sadinle2016LeastAS, romano2019conformalized} for further discussion of this point in that context.

\subsection{Optimality in a more general setting}
Next, we characterize the set-valued predictor that leads to the smallest sets for a wider class of losses. Suppose our loss takes the form
\begin{equation}
    L(y; \S) = \int_{z\in \S^c}\ell(y, z)d\mu(z),
\end{equation}
for some nonnegative $\ell$ and a finite measure $\mu$. The function $\ell$ measures the cost of not including $z$ in the prediction set when true response is $y$. For instance, $\ell(y, z) = L_y \I(y = z)$ and $\mu$ is the counting measure in the case considered above. 
Then the optimal $\T_\lambda$ is given by
\begin{equation}
\T_\lambda(x) = \{z: \E[\ell(Y; z)\mid X = x]\ge -\lambda\},
\label{eq:optimal_T_general}
\end{equation}
for $\lambda \in \Lambda \subset (-\infty, 0]$, as stated next.
\begin{theorem}[Optimality of set predictors, generalized form]
In the setting above, let $\T':\mathcal{X} \to \Y'$ be any set-valued predictor such that $R(\T') \leq R(\T_\lambda)$, where $\T_\lambda$ is given by \eqref{eq:optimal_T_general}.
Then,
\begin{equation}
\E[|\T_{\lambda}(X)|] \leq \E[|\T'(X)|].
\end{equation}
\label{thm:unconditional-optimality-ext}
\vspace{-0.5cm}
\end{theorem}
For the case considered in Section~\ref{sec:unconditional-oracle}, $\E[\ell(Y; z)\mid X = x] = L_{z}p(z \mid x)$, so we see Theorem~\ref{thm:unconditional-optimality-ext} includes Theorem~\ref{thm:unconditional-optimality} as a special case. As before, in practice we must estimate the distribution of $Y$ given $X$ from data, so we would not typically be able to implement this predictor exactly. Moreover, even if we perfectly knew the distribution of $Y_i$ given $X = x$, the sets in \eqref{eq:optimal_T_general} may not be easy to compute. Nonetheless, it is encouraging that we can understand the optimal set predictor for this important set of losses.

\section{Examples}
\label{sec:experiments}

Next, we apply our proposed method to five prediction problems. 
For each task, we introduce a relevant loss function and set-valued predictor, and then evaluate the performance of UCB calibration. 
The reader can reproduce our experiments using our \href{https://github.com/aangelopoulos/rcps}{\textcolor{blue}{public GitHub repository}}.

\subsection{Classification with a class-varying loss}
\label{sec:class-varying}
\begin{figure}[t]
    \centering
    \includegraphics[width=\linewidth]{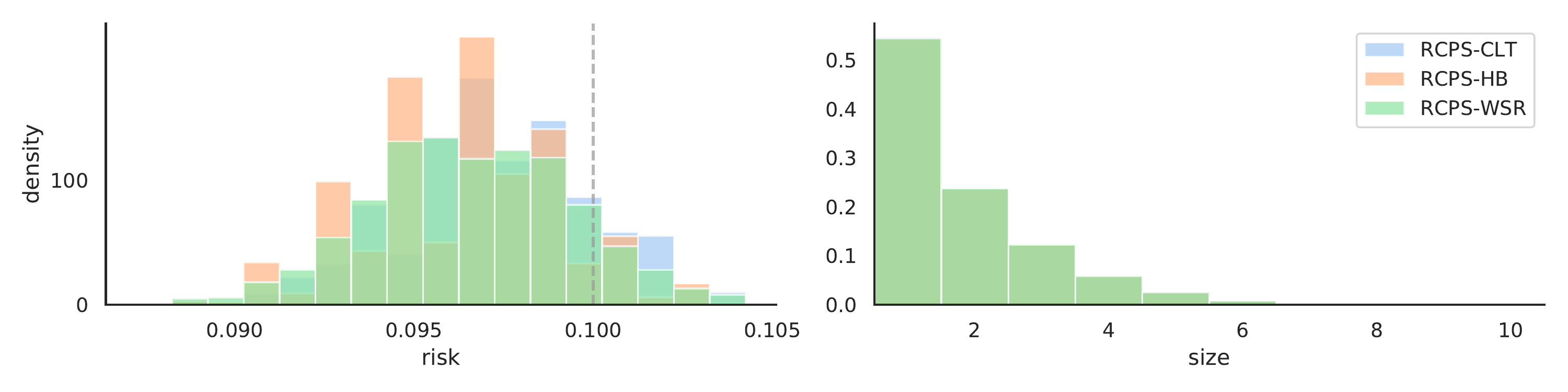}
    \vspace{-1cm}
    \caption{{\bf Prediction set results on Imagenet.} The risk and set sizes for an RCPS are plotted as histograms over 100 different random splits of Imagenet, with parameters $\alpha=0.1$ and $\delta=0.1$. For details see Section~\ref{sec:class-varying}. The set sizes for all three methods overlap.}
    \label{fig:histograms_imagenet}
\end{figure}

Suppose each observation has a single correct label $y$, and each label incurs a different, fixed loss if it is not correctly predicted:
\begin{equation*}
L(y, \S) = L_y \I_{\{y \notin \S\}}. 
\end{equation*}
This was the setting of our oracle result in Section~\ref{sec:unconditional-oracle}, and the medical diagnostic setting from the introduction also has this form. 
We would like to predict a set of labels that controls this loss.
Towards that end, we define the family of nested sets
\begin{equation}
    \label{eq:nested-imagenet}
    \T_{\lambda}(x) = \big\{ y : \hat{\pi}_x(y) > -\lambda\big\},
\end{equation}
where $\hat{\pi}_x : \Y \to [0,1]$ represents a classifier, usually designed to estimate $P(Y|X)$.
This family of nested sets simply returns the set of classes whose estimated conditional probability exceeds the value $-\lambda$, as in Figure~\ref{fig:riskscore}. (The negative on $\lambda$ comes from the definition of nesting, which asks sets to grow as $\lambda$ grows.)

Here, we conduct an experiment on Imagenet---the gold-standard computer vision classification dataset---comprised of one thousand classes of natural images~\cite{deng2009imagenet}.
For this experiment, we assign the loss $L_{y}$ of class $y \in \{1,...,1000\}$ as $L(y)\overset{i.i.d.}{\sim}{\rm Unif}(0,1)$. 
We use a pretrained ResNet-152 from the \texttt{torchvision} repository as the base model $\hat{\pi}_x$~\cite{marcel2010torchvision,he2016deep}.
We then choose $\lhat$ as in Theorem~\ref{thm:WSR}.
Figure~\ref{fig:histograms_imagenet} summarizes the performance of our prediction sets over 100 random splits of Imagenet-Val with 30,000 points used for calibration and the remaining 20,000 used for evaluation.
The RCPS procedure controls the risk at the correct level and the sets have reasonable sizes.

\subsection{Multi-label classification}
\label{sec:multilabel}

\begin{figure}[t]
    \centering
    \includegraphics[width=\linewidth]{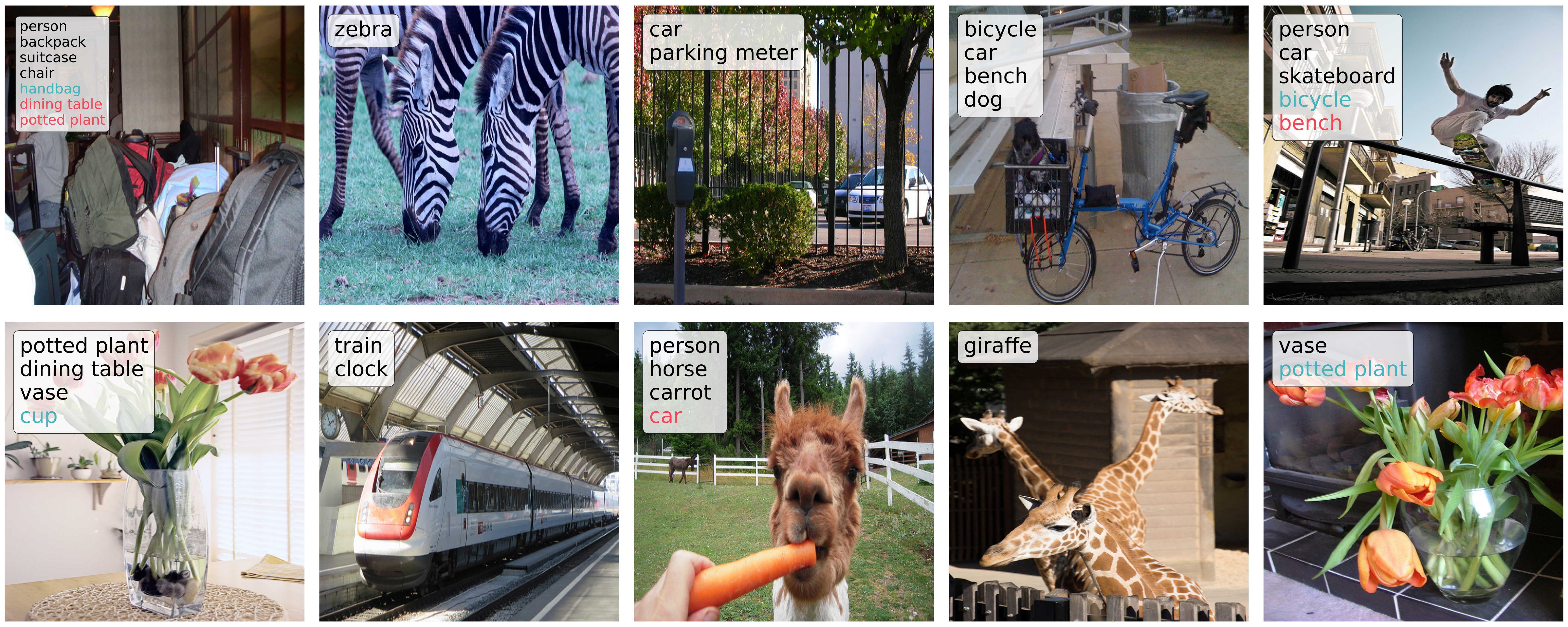}
    \vspace{-0.5cm}
    \caption{{\bf Multi-label prediction set examples on MS COCO.} Black classes are correctly identified (true positives), blue ones are spurious (false positives), and red ones are missed (false negatives). }
    \label{fig:grid_coco_multiclass}
    \includegraphics[width=\linewidth]{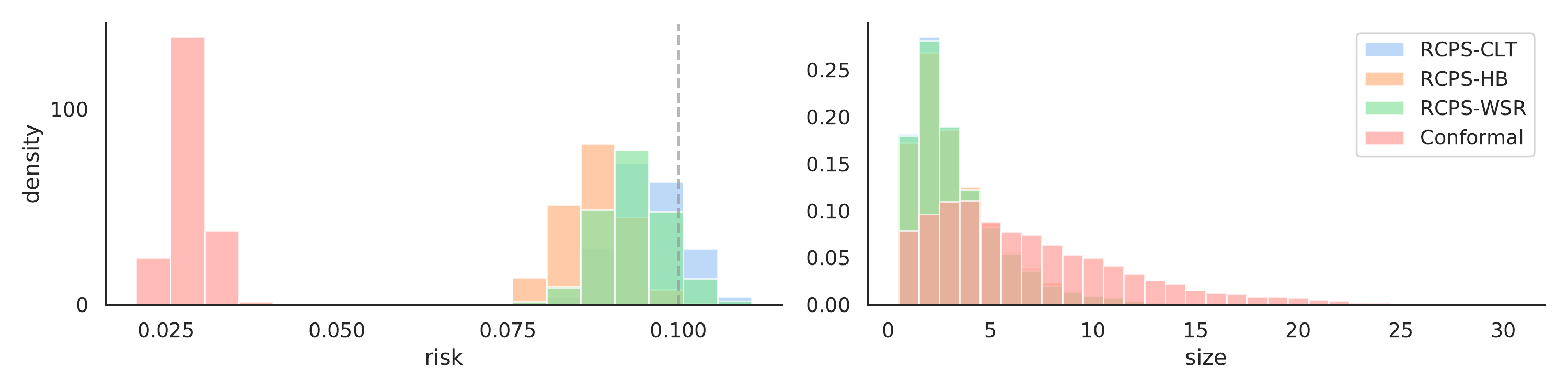}
    \vspace{-1cm}
    \caption{{\bf Multi-label prediction set results on MS COCO.} The risk and set sizes are plotted as histograms over 1000 different random splits of MS COCO, with parameters $\alpha=0.1$ and $\delta=0.1$. We also include a conformal baseline. For details see Section~\ref{sec:multilabel}.}
    \label{fig:histograms_coco_multiclass}
\end{figure}

Next, we consider the multi-label classification setting where each observation may have multiple corresponding correct labels; i.e., the response $y$ is a subset of $\{1,...,K\}$. 
Here, we seek to return prediction sets that control the loss
\begin{equation}
\label{eq:multiclass-loss}
L(y, \S) = 1 - \frac{|y \cap \S|}{|y|}
\end{equation}
at level $\alpha$. That is, we want to capture at least a $1-\alpha$ proportion of the correct labels for each observation, on average.
In this case, our nested sets
\begin{equation}
    \T_\lambda(x) = \big\{z \in \{1,...,K\} :\hat{\pi}_x(z) > -\lambda\big\}
\end{equation}
depend on a classifier $\hat{\pi}_x$ that does not assume classes are exclusive, so their conditional probabilities generally do not sum to $1$. 
Note that in this example we choose the output space $\Y'$ to be $\Y = 2^{\{1,\dots,K\}}$ (rather than $2^\Y$ as was done our previous example), since here $\Y$ is already a suitable space of sets.

To evaluate our method, we use the Microsoft Common Objects in Context (MS COCO) dataset, a large-scale, eighty-category dataset of natural images in realistic and often complicated contexts~\cite{lin2014microsoft}.
We use TResNet as the base model, since it has the state-of-the-art classification performance on MS COCO at the time of writing~\cite{ridnik2020tresnet}.
The standard procedure for multi-label estimation in computer vision involves training a convolutional neural network to output the vector of class probabilities, and then thresholding the probabilities in an ad-hoc manner return a set-valued prediction.
Our method follows this general approach, but rigorously chooses the threshold so that the risk is controlled at a user-specified level $\alpha$, which we take to be 10\%. To set the threshold, we choose $\lhat$ as in Theorem~\ref{thm:WSR} using 4,000 calibration points, and then we evaluate the risk on an additional test set of 1,000 points.
In Figure~\ref{fig:grid_coco_multiclass} we report on our our method's performance on ten randomly selected images from MS COCO, and in Figure~\ref{fig:histograms_coco_multiclass} we quantitatively summarize the performance of our prediction sets.
Our method controls the risk and gives sets with reasonable sizes.

In this setting, it is also possible to consider a conformal prediction baseline. To frame this problem in a way such that conformal prediction can be used, we follow \cite{cauchois2020knowing} and say that a test point is covered correctly if $y \subset T(x)$ and miscovered otherwise. That is, a point is covered only if the prediction set contains all true labels. The conformal baseline then uses the same set of set-valued predictors as above, but chooses the threshold as in \cite{cauchois2020knowing} so that there is probability $1-\alpha$ that all of the labels per image are correctly predicted. In Figure~\ref{fig:grid_coco_multiclass}, we find that the conformal baseline returns larger prediction sets. The reason is that the notion of coverage used by conformal prediction is more strict, requiring that all classes are covered. By contrast, the RCPS method can incorporate less brittle loss functions, such as the false negative rate in \eqref{eq:multiclass-loss}.
 
\subsection{Hierarchical classification}
\label{sec:hierarchical-imagenet}

\begin{figure}[t]
    \centering
    \includegraphics[width=\linewidth]{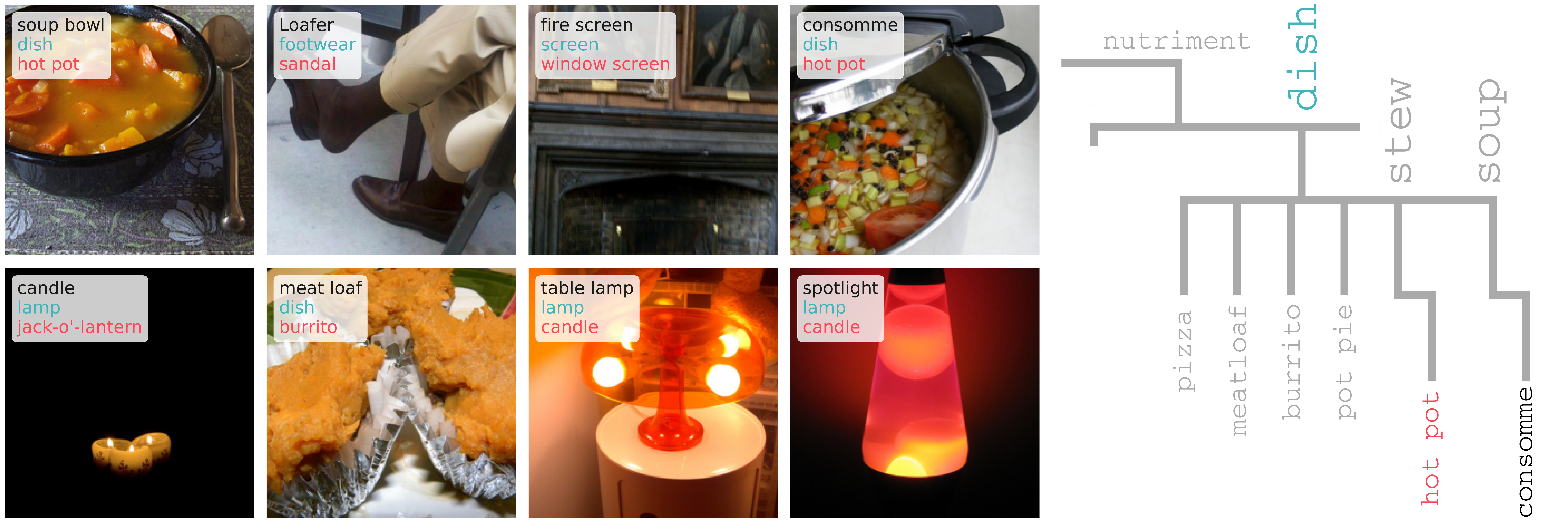}
    \vspace{-0.5cm}
    \caption{{\bf Hierarchical predictions.} We show randomly selected examples of hierarchical prediction sets on Imagenet where the point prediction is incorrect but the prediction sets cover the true label. The black label is the ground truth class, the blue label is our prediction, and the red label is the top-1 output of a ResNet-18. Our prediction is an ancestor in the WordNet hierarchy of both the true class and the model's top-1 prediction. See the rightmost panel for an example subtree from the WordNet hierarchy.}
    \label{fig:hierarchical-grid}
    \includegraphics[width=\linewidth]{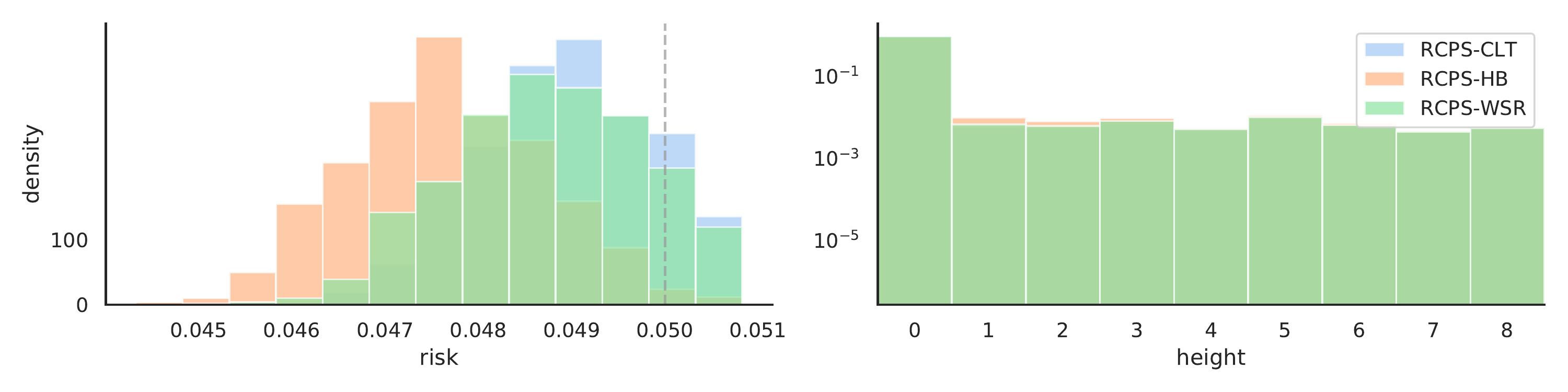}
    \vspace{-1cm}
    \caption{{\bf The risk and height of RCPS for hierarchical classification.} We show histograms of risk and height (distance from the leaf node) over 100 different random splits of the Imagenet dataset, with parameters $\alpha=0.05$ and $\delta=0.1$. For details see Section~\ref{sec:hierarchical-imagenet}.}
    \label{fig:hierarchical-histograms}

\end{figure}

Next, we discuss the application of RCPS to prediction problems where there exists a hierarchy on $K$ labels. Here, we have a response variable $y \in \{1,\dots,K\}$ with the structure on the labels encoded as a tree with nodes $V$ and edges $E$ with a designated root node, finite depth $D$, and $K$ leaves, one for each label. To represent uncertainty while respecting the hierarchical structure, we seek to predict a node $\hat{y} \in V$ that is as precise as possible, provided that that is an ancestor of $y$. Note that with our tree structure, each $v \in V$ can be interpreted as a subset of $\{1,\dots,K\}$ by taking the set of all the leaf-node descendants of $v$, so this setting is a special case of the set-valued prediction studied in this work.

We now turn to a loss function for this hierarchical label structure.
Let $d:V \times V \to \mathbb{Z}$ be the function that returns the length of the shortest path between two nodes, let $\mathcal{A}: V \to 2^V$ be the function that returns the ancestors of its argument, and let $\mathcal{P}: V \to 2^V$ be the function that returns the set of leaf nodes that are descendants of its argument. Further define a hierarchical distance 
\begin{equation*}
d_H(v,u) = \underset{a \in \mathcal{A}(v)}{\inf} \{ d(a,u) \}.
\end{equation*}
For a set of nodes $\S \in 2^V$, we then define the set-valued loss 
\begin{equation}
    \label{eq:hierarchical-loss}
    L(y,\S) = \underset{s \in \S}{\inf} \{ d_H(y,s) \} / D.
\end{equation}
This loss returns zero if $y$ is a child of any element in $\S$, and otherwise returns the minimum distance between any element of $\S$ and any ancestor of $y$, scaled by the depth $D$.

Lastly, we develop set-valued predictors that respect the hierarchical structure.
Define a model $\hat{f}:\X \to [0,1]^K$ that outputs an estimated probability for each class.
For any $x \in \X$, let $\hat{y}(x) = \arg\max_{k} \hat{f}(x)_k$ be the class with highest estimated probability.
We also let $g(v,x) = \underset{ k \in \mathcal{P}(v) }{\sum}\hat{f}(x)_k$ be the sum of scores of leaves descended from $v$.
Then, we choose our family of set-valued predictors as:
\begin{equation}
    \label{eq:hierarchical-sets}
    \T_\lambda(x) = \underset{\{a \in \mathcal{A}(\hat{y}(x)) \ : \ g(a,x) \geq -\lambda\}}{\bigcap}  \mathcal{P}(a).
\end{equation}
In words, we return the leaf nodes of the smallest subtree that includes $\hat{y}(x)$ that has estimated probability mass of at least $-\lambda$. This subtree has a unique root $v \in V$, so can equivalently view $\T_\lambda(x)$ as returning the node $v$.

We return to the Imagenet dataset for our empirical evaluations.
The Imagenet labels form a subset of the WordNet hierarchy~\cite{fellbaum2012wordnet}, and we parsed them to form the tree.
Our results are akin to those of ~\cite{deng2012hedging}, although their work does not have distribution-free statistical guarantees and instead takes an optimization approach to the problem.
The maximum depth of the WordNet hierarchy is $D=14$.
Similarly to Section~\ref{sec:class-varying}, we used a pretrained ResNet-18 from the \texttt{torchvision} repository as the base model for Algorithm~\ref{alg:greedy}, and chose $\hat{\lambda}$ as in Theorem~\ref{thm:WSR}.
Figure~\ref{fig:hierarchical-grid} shows several examples of our hierarchical predictions on this dataset, and Figure~\ref{fig:hierarchical-histograms} summarizes the performance of the predictor. As before, we find that RCPS controls the risk at the desired level, and the predictions are generally relatively precise (i.e., of low depth in the tree).
 
\subsection{Image segmentation}
\label{sec:segmentation}

\begin{figure}[t]
    \centering
    \includegraphics[width=\linewidth]{figures/multipolyp_grid_fig.pdf}
    \vspace{-0.5cm}
    \caption{{\bf Polyp segmentations.} We show examples of polyps along with prediction sets that capture 90\% of the true polyp pixels per polyp per image, generated with our method using the CLT bound. White pixels are correctly identified polyp pixels (true positives), blue ones are spurious (false positives), and red ones are missed (false negatives). The top two rows show examples with a single polyp per image, and the second two rows show examples with two polyps per image.}
    \label{fig:grid_polyp}
    \includegraphics[width=\linewidth]{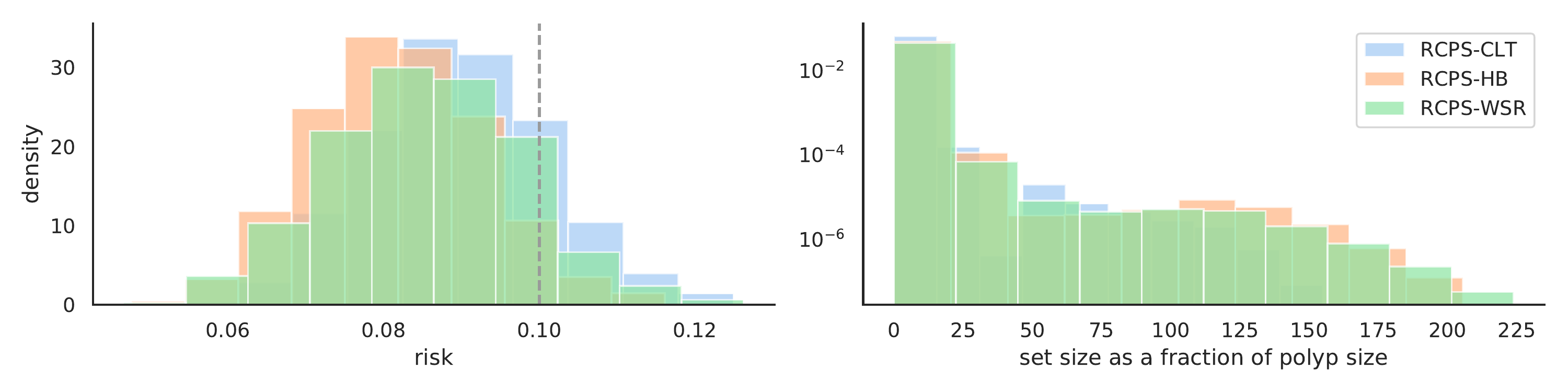}
    \vspace{-1cm}
    \caption{{\bf Polyp segmentation results.} The risk and normalized set size are plotted as histograms over different random splits of the polyp dataset, with parameters $\alpha=0.1$ and $\delta=0.1$. For details see Section~\ref{sec:segmentation}.}
    \label{fig:histograms_polyp}
\end{figure}

In the binary segmentation setting, we are given an $d_1 \times d_2 \times c$-dimensional image $x \in \R^{d_1 \times d_2 \times c}$ and seek to predict a set of object pixels $y \subseteq \G$, where $\G = \big\{(i,j) : 1 \leq i \leq d_1, 1 \leq j \leq d_2 \big\}$.
Intuitively, $y$ is a set of pixels that differentiates objects of interest from the backdrop of the image.

Using the technique in Section~\ref{sec:multilabel}, one may easily return prediction sets that capture at least a $1-\alpha$ proportion of the object pixels from each image with high probability.
However, if there are multiple objects in the image, we may want to ensure our algorithm does not miss an entire, distinct object.
Therefore, we target a different goal: returning prediction sets that capture a $1-\alpha$ fraction of the object pixels \emph{from each object} with high probability.
Specifically, consider $h: \Y \to 2^\Y$ to be an 8-connectivity connected components function~\cite{hirschberg1979computing}.
Then $h(y)$ is a set of distinct regions of object pixels in the input image. For example, in the bottom right image of Figure~\ref{fig:grid_polyp}, $h(y)$ would return two subsets of $\G$, one for each connected component. With this notation, 
we want to predict sets of pixels $\S \subseteq \G$ that control the proportion of missed pixels per object:
\begin{equation}
    \label{eq:segmentation_loss}
    L(y, \S) = \frac{\underset{y' \in h(y)}{\sum}|y' \setminus \S| / |y'|}{|h(y)|}.
\end{equation}
With this loss, if there are regions of different sizes, we would still incur a large loss for missing an entire small region, so this loss better captures our goal in image segmentation.

Having defined our loss, we now turn to our set construction.
Standard object segmentation involves a model $\hat{f} : \R^{d_1 \times d_2} \to [0,1]^{d_1 \times d_2}$ that outputs approximate scores (e.g., after a sigmoid function) for each pixel in the image, then binarizes these scores with some threshold.
To further our goal of per-object validity, in this experiment we additionally detect local peaks in the raw scores via morphological operations and connected components analysis, then re-normalize the connected regions by their maximum value. 
We will refer to this renormalization function as $r : [0,1]^{d_1 \times d_2} \to [0,1]^{d_1 \times d_2}$, and describe it precisely in Appendix~\ref{app:polyp-renorm}. We choose our family of set-valued predictors as
\begin{equation}
    \label{eq:polyp-sets}
    \T_{\lambda} = \big\{ (i,j) : r(\hat{f}(x))_{i,j} \geq -\lambda \big\},
\end{equation}
and then select $\hat{\lambda}$ as in Theorem~\ref{thm:WSR} or as in Theorem~\ref{thm:clt_coverage}.

We evaluated our method with an experiment combining several open-source polyp segmentation datasets: Kvasir~\cite{pogorelov2017kvasir}, Hyper-Kvasir~\cite{borgli2020hyperkvasir}, CVC-ColonDB and CVC-ClinicDB~\cite{bernal2012towards}, and ETIS-Larib~\cite{silva2014toward}.
Together, these datasets include 1,781 examples of segmented polyps, and in each experiment we use 1,000 examples for calibration and the remainder as a test set.
We used PraNet~\cite{fan2020pranet} as our base segmentation model. In Figure~\ref{fig:grid_polyp} we report on our method's performance on $20$ randomly selected images from the polyp datasets that contain at least two polyps, and in Figure~\ref{fig:histograms_polyp} we summarize the quantitative performance of our prediction sets. RCPS again control the risk at the desired level, and the average prediction set size size is comparable to the average polyp size.

\subsection{Protein structure prediction}
\label{sec:proteins}
\begin{figure}[t]
    \centering
    \includegraphics[width=\linewidth]{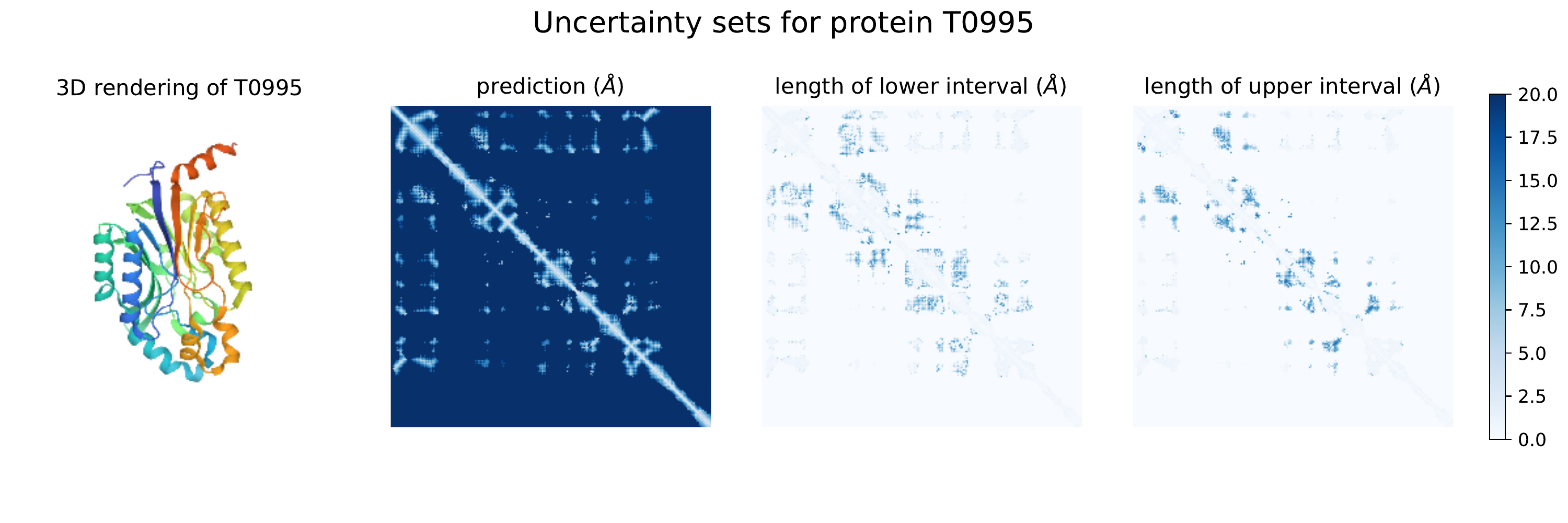}
    \vspace{-1.5cm}
    \caption{{\bf Protein distograms.} We show AlphaFold's predicted distances between residues of protein T0995 along with prediction sets at $\alpha=2\AA$ and $\delta=0.1$. The prediction set for the whole protein is the union of distance intervals for each pair of residues, and the right two panels report the distance from the point prediction to the lower and upper endpoints for each of these intervals.}
    \label{fig:grid_proteins}
    \includegraphics[width=\linewidth]{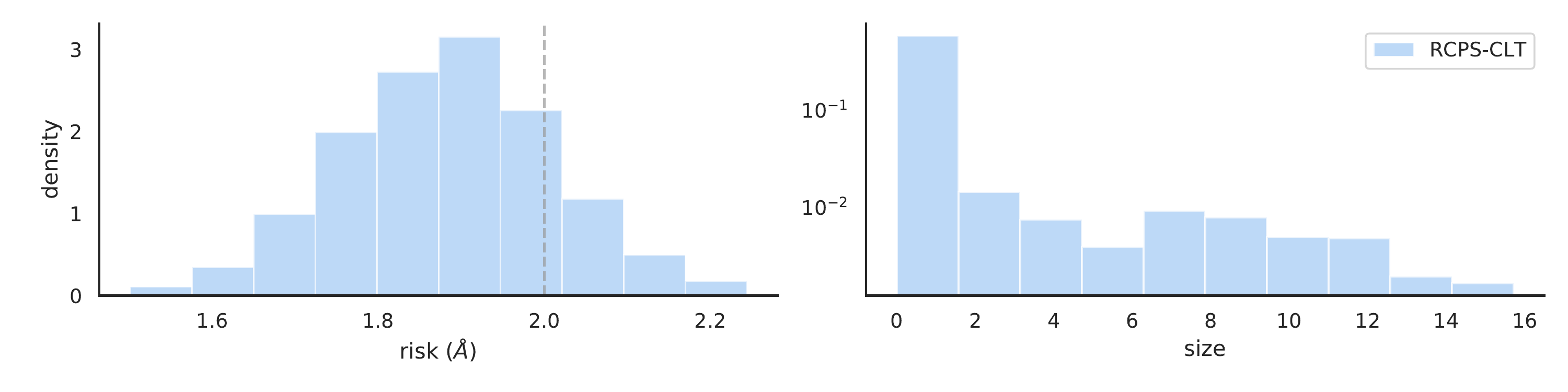}
    \vspace{-1cm}
    \caption{{\bf Protein structure prediction results.} The risk in $\AA$ and interval size (pooling all entries of each distogram) in $\AA$ are plotted as histograms, repeating for many random splits of the CASP-13 test-set. }
    \label{fig:histograms_proteins}
\end{figure}
We finish the section by demonstrating RCPS for protein structure prediction, inspired by the recent success of AlphaFold.
{\em Proteins} are biomolecules comprising one or more long chains of amino acids; when amino acids form a chemical bond to form a protein, they eject a water molecule and become amino acid {\em residues}.
Each amino acid residue has a common amine-carboxyl backbone and a different {\em side chain} with electrical and chemical properties that together determine the 3D conformation of the whole protein, and thus its function.
The so-called {\em protein structure prediction problem} is to predict a protein's three dimensional structure from a list of its residues.
A critical step in AlphaFold's protein structure prediction pipeline involves predicting the distance between the {\em $\beta$-carbons} (the second-closest carbon to the side-chain) of each residue.
These distances are then used to determine the protein's 3D structure.
We express uncertainty directly on the distances between $\beta$-carbons.

Concretely, consider the alphabet $\Sigma=\{A,C,D,E,F,G,H,I,K,L,M,N,P,Q,R,S,T,V,W,Y\}$, where each letter is the common abbreviation for an amino acid (for example, $A$ denotes Alanine).
The feature space consists of all possible words over $\Sigma$, commonly denoted as $\X=\Sigma^*$.
The label space $\Y$ is the set of all symmetric matrices with positive elements of any side length.
In an example $(x,y) \in \X \times \Y$, the entry $y_{i,j}$ defines the distance in 3D space of residues $x_i$ and $x_j$; hence, $y \in \R^{|x|\times |x|}$, and $y_{i,j}=y_{j,i}$.
We seek to predict sets $\S$ that control the $\ell_1$ projective distance from $y$ to $\S$:
\begin{equation}
    \label{eq:protein-loss}
    L(y,\S) = \underset{s \in \S}{\inf} \left\{ \frac{1}{|x|^2}\underset{i,j}{\sum}|y_{i,j}-s_{i,j}|\right\}.
\end{equation}

Now we turn to the set construction, which we specialize to the AlphaFold pipeline.
Because the AlphaFoldv2 codebase was not released at the time this paper was written, we use AlphaFoldv1 here~\cite{senior2020improved}.
For a residue chain $x \in \X$, consider a variadic function $h(x) \in [0,1]^{|x| \times |x| \times K}$, where $K$ is a positive integer and $\underset{k}{\sum}h(x)_{i,j,k}=1$ for all fixed choices of $i$ and $j$.
The function $h$ represents a probability distribution over distances $d_1,...,d_K$ for each distance between residues as a histogram; the output of $h$ is referred to as a {\em distogram}.
Given a distogram, we construct the family of set valued predictors
\begin{equation}
    \label{eq:protein-sets}
    \T_{\lambda}(x) = \prod_{0 \le i,j \le |x|}\big\{ d_k : h(x)_{i,j,k} \geq -\lambda \big\} \; \;
\end{equation}
and choose $\lhat$ as in \eqref{eq:lambda_hat_def}, as usual.

We evaluated our set construction algorithm on the $71$ test points from the CASP-13 challenge on which DeepMind released the output of their model.
In the AlphaFoldv1 pipeline, $K=64$ and $d_1,...,d_k=2\AA,...,20\AA$.
Since the data prepossessing pipepline was not released, no ground truth distance data is available. Instead, we generated semi-synthetic data points by sampling once from the distogram corresponding to each protein.
We choose parameters $\alpha=2\AA$ and $\delta=0.1$, and, due to the small sample size ($35$ calibration and $36$ test points), we only report results using the CLT bound, because the exact concentration results are hopelessly conservative with only $35$ calibration points.

Figure~\ref{fig:grid_proteins} shows an example of our prediction sets on protein T0995 (PDB identifier 3WUY) \cite{zhang2014structural}.
Figure~\ref{fig:histograms_proteins} shows the quantitative performance of the CLT, which nearly controls the risk.
The strong performance of the CLT in this small-sample regime is encouraging and suggests that our methodology can be applied to problems even with small calibration sets.

\section{Other Risk Functions}
Thus far we have defined the risk to be the mean of the loss on a single test point. In this section we consider generalizations to a broader range of settings in which the risk function is a functional other than the mean and/or the loss is a function of multiple test points. To this end, recall that there are two mathematical ingredients to the UCB calibration framework. First, there is a family of possible predictors such that the notion of error is monotone in the parameter indexing the family, $\lambda \in \Lambda$. Second, for each element $\lambda$, we have a pointwise concentration result that gives the upper confidence bound for the error at that $\lambda$. With these two ingredients, we carry out UCB calibration by selecting $\hat{\lambda}$ as in \eqref{eq:lambda_hat_def}, which has error-control guarantees as in Theorem~\ref{thm:abstract_control}.
We demonstrate the scope of this more general template with a few examples.

\subsection{Uncertainty quantification for ranking}
\label{subsec:ranking}

We consider the problem of uncertainty quantification for learning a ranking rule \cite[see, e.g.,][]{clemencon2008}. We assume we have an i.i.d.\ sequence of points, $(X_1,Y_1),\dots, (X_m, Y_m)$, where $Y \in \{1,\dots,k\}$. We wish to learn a \emph{ranking rule}: $r: \X \times \X \to \R$ such that $r(X_i, X_j)$ tends to be positive when $Y_i > Y_j$ and tends to be negative otherwise. Given a ranking rule $\hat{r} : \X \times \X \to \R$ that has been estimated based on the data $(X_{n+1}, Y_{n+1}), \dots, (X_m, y_m)$, we consider calibrating this ranking rule to control loss based on $(X_1,Y_1),\dots,(X_n, Y_n)$.

To quantify uncertainty in this setting, we use a set-valued ranking rule $\T_\lambda : \X \times \X \to 2^ \R$. Here, higher uncertainty is encoded by returning a larger set. We assume that we have a family of such predictors satsifying the following monotonicity property:
\begin{equation}
\lambda < \lambda' \implies \T_\lambda(x_1, x_2) \subset \T_{\lambda'}(x_1, x_2).
\label{eq:ustat2_nested_predictions}
\end{equation}
For example, we could take $\T_\lambda(x_1, x_2) = (\hat{r}(x_1, x_2) - \lambda, \hat{r}(x_1, x_2) + \lambda)$ for $\lambda \ge 0$.
Our notion of error control here is that we wish to correctly determine which response is larger, so we use the following error metric:
\begin{equation*}
    L(y_1, y_2, \S) = \ind{\sup{\S} < 0}\ind{y_1 > y_2} + \ind{\inf{\S} > 0}\ind{y_1 < y_2},
\end{equation*}
which says that we incur loss one if the prediction $\S$ contains no values of the correct sign and zero otherwise. More generally, we could use any loss function with the following nesting property:
\begin{equation}
    \S \subset \S' \implies L(y_1, y_2, \S) \ge L(y_1, y_2, \S').
    \label{eq:ustat2_nested_loss}
\end{equation}
We then define the risk as 
\begin{equation}
R(\T_\lambda) = \E[L(Y_1, Y_2, \T_\lambda(X_1, X_2)],
\label{eq:ustat2_risk}
\end{equation}
which can be estimated via its empirical version on the holdout data:
\begin{equation}
\widehat{R}(\T_\lambda) = \sum_{1 \le i < j \le n} L(Y_i, Y_j, \T_\lambda(X_i, X_j)).
\label{eq:ustat2_emp_risk}
\end{equation}
Suppose additionally that we have an upper confidence bound for $R(\T_\lambda)$ for each $\lambda$, as in \eqref{eq:general_bound}. In this setting, we can arrive at such an upper bound using the concentration of U-statistics, such as the following result.
\begin{prop}[Hoeffding--Bentkus--Maurer inequality for bounded U-statistics of order two]
Consider the setting above with any loss function $L$ bounded by one. Let $m = \lfloor n / 2\rfloor$. Then, for any $t\in (0, R(\lambda))$,
\begin{align*}
P(\Rhat(\lambda)\le t) \le g^{{\rm U}}(t; R(\lambda)) \triangleq \min\bigg(& \exp\left\{-mh_1(t; R(\lambda))\right\}, eP\left({\rm Binom}(m; R(\lambda)) \leq \lceil mt\rceil\right)\nonumber\\
& \inf_{\nu > 0}\exp\left\{-\frac{n\nu}{2} \left(\frac{R(\lambda)}{1 + 2G(\nu)} - t\right)\right\}\Bigg),
\end{align*}
where $G(\nu) = (e^{\nu} - \nu - 1) / \nu$.
\label{prop:ustat2_ucb}
\end{prop}

With this upper bound, we can implement UCB calibration by selecting $\hat{\lambda}$ through Proposition \ref{prop:generic}. This gives a finite-sample, distribution-free guarantee for error control of the uncertainty-aware ranking function $\T_{\hat{\lambda}}$:
\begin{theorem}[RCPS for ranking]
\label{thm:ranking}
Consider the setting above with any loss function bounded by one. Then, with probability at least $1-\delta$, we have $R(\T_{\hat{\lambda}}) \le \alpha$.
\end{theorem}

This uncertainty quantification is natural; as $\lambda$ grows, the set $\T_\lambda(X_i, X_j)$ will more frequently include both positive and negative numbers, in which case the interpretation is that our ranking rule $\T_\lambda$ abstains from ranking those two inputs. The UCB calibration tunes $\lambda$ so that we abstain from as few pairs as possible, while guaranteeing that the probability of making a mistake on inputs for which we do not abstain is below the user-specified level $\alpha$.

\subsection{Uncertainty quantification for metric learning}
We next consider the problem of supervised metric learning, where we have an i.i.d.\ sequence of points $(X_1,Y_1),\dots, (X_m, Y_m)$, with $\Y = \{1,\dots,k\}$. We wish to train a {\em metric} $d : \X \times \X \to \R$ such that it separates the classes well. That is, we wish for $d(X_i, X_j)$ to be small for points such that $Y_i = Y_j$ and large otherwise. We assume that we have fit a metric $\hat{d}$ based on the data $(X_{n+1}, Y_{n+1}), \dots, (X_m, Y_m)$ and our goal is to calibrate this metric based on $(X_1,Y_1),\dots,(X_n, Y_n)$. Our development will closely track the ranking example, again leveraging U-statistics of order two.

To formulate a notion of uncertainty quantification for metric learning, we express uncertainty by introducing a set-valued metric,$\T_\lambda : \X \times \X \to 2^\R$, where greater uncertainty is represented by returning a larger subset of $\R$. We assume that we have a family of such set-valued metrics that have the monotonicity property in \eqref{eq:ustat2_nested_predictions}.
For example, we could take $\T_\lambda(x_1, x_2) = (\hat{d}(x_1, x_2) - \lambda, \hat{d}(x_1, x_2) + \lambda)$ for $\lambda \ge 0$.
To formalize our goal that the classes be well separated, we take as our loss function the following:
\begin{equation*}
    L(y_1, y_2, \S) = (\inf(\S) - 1)^+ \ind{y_1 = y_2} + (\sup(\S) - 1)^- \ind{y_1 \ne y_2},
\end{equation*}
where $\S$ is a set-valued prediction of the distance between $x_1$ and $x_2$. This choice implies that we take a distance of one to be the decision boundary between classes, so that points with distance less than one should correspond to the same class. We incur an error if two points in the same class are predicted to have distance above one. This particular parameterization is somewhat arbitrary, and we could instead take any loss satisfying the nesting property in \eqref{eq:ustat2_nested_loss}.
We then define the risk as in \eqref{eq:ustat2_risk} and the empirical risk as in \eqref{eq:ustat2_emp_risk}.
From here we can adopt the upper bound from Proposition~\ref{prop:ustat2_ucb} if we additionally restrict $\hat{d}$ to return values in a bounded set. We again implement UCB calibration by selecting $\hat{\lambda}$ as in \eqref{eq:lambda_hat_def}, which yields the following guarantee.
\begin{theorem}[RCPS for metric learning]
\label{thm:metric_learning}
In the setting above, suppose the loss function is bounded by $1$. Then, with probability at least $1-\delta$, we have $R(\T_{\hat{\lambda}}) \le \alpha$.
\end{theorem}

\subsection{Adversarially robust uncertainty quantification}
Finally, we briefly remark how our framework might be extended to handle uncertainty quantification with adversarial robustness \cite[see, e.g.,][]{madry2018, carmon2019unlabeled}. In this setting, the goal is to fit a model that performs well even for the worst-case perturbation of each input data point over some limited set of perturbations, such as an $\ell^\infty$ ball. This notion of robust loss can be translated into our framework by defining the appropriate risk function. For example, we could consider the risk function
\begin{equation*}
    R^{\textnormal{(rob)}}(\T) = \E\left[\sup_{x' \in \mathcal{B}_\epsilon(X)} L(Y, \T(x')) \right],
\end{equation*}
where $\mathcal{B}_\epsilon(X)$ is an $\ell^\infty$ ball of radius $\epsilon$ centered at $X$. For a family of set-valued functions $\{\T_\lambda\}_{\lambda \in \Lambda}$, one can estimate the risk on a holdout set and then choose the value of $\lambda$ with the UCB calibration algorithm, resulting in a finite-sample guarantee on the risk. While carrying out this procedure would require computational innovations, our results establish that it is statistically valid.

\section{Discussion}
Risk-controlling prediction sets are a new way to represent uncertainty in predictive models. Since they apply to any existing model without retraining, they are straightforward to use in many situations. Our approach is closely related to that of split conformal prediction, but is more flexible in two ways. First, our approach can incorporate many loss functions, whereas conformal prediction controls the coverage---i.e, binary risk. The multilabel classification setting of Section~\ref{sec:multilabel} is one example where RCPS enables the use of a more natural loss function: the false negative rate.
Second, risk-controlling prediction sets apply whenever one has access to a concentration result, whereas conformal prediction relies on exchangeability, a particular combinatorial structure. Concentration is a more general tool and can apply to a wider range of problems, such as the uncertainty quantification for ranking presented in Section~\ref{subsec:ranking}. To summarize, in contrast to the standard train/validation/test split paradigm which only estimates global uncertainty (in the form of overall prediction accuracy), RCPS allow the user to automatically return \emph{valid instance-wise uncertainty estimates} for many prediction tasks.

\section*{Acknowledgements}

We wish to thank Emmanuel Cand\`es, Maxime Cauchois, Edgar Dobriban, Todd Chapman, Mariel Werner, and Suyash Gupta for giving feedback on an early version of this work.
A.~A.~was partially supported by the National Science Foundation Graduate Research Fellowship Program and a Berkeley Fellowship. This work was partially supported by the Army Research
Office under contract W911NF-16-1-0368.

\printbibliography

\clearpage
\appendix

\section{Proofs}
\label{app:proofs}

\begin{theorem}[Validity of UCB calibration, abstract form]
    \label{thm:abstract_control_general}
    Let $R : \Lambda \to \mathbb{R}$ be a continuous monotone nonincreasing function such that $R(\lambda) \le \alpha$ for some $\lambda \in \Lambda$.
    Suppose $\Rhat^+(\lambda)$ is a random variable for each $\lambda \in \Lambda$ such that \eqref{eq:general_bound} holds pointwise.
    Then for $\hat{\lambda}$ chosen as in \eqref{eq:lambda_hat_def},
    \begin{equation}
    P\left(R(\lambda) \le \alpha\right) \ge 1 - \delta.
    \end{equation}
\end{theorem}

\begin{proof}[Proof of Theorem~\ref{thm:abstract_control_general}]
Consider the smallest $\lambda$ that controls the risk:
\begin{equation*}
    \lambda^* \triangleq \inf\{\lambda \in \Lambda  : R(\lambda) \le \alpha\}.
\end{equation*}
Suppose $R(\lhat) > \alpha$. By the definition of $\lambda^*$ and the monotonicity and continuity of $R(\cdot)$, this implies
$\lhat < \lambda^*$. By the definition of $\lhat$, this further implies that $\Rhat^+(\lambda^*) < \alpha$. But since $R(\lambda^*) = \alpha$ (by continuity) and by the coverage property in \eqref{eq:general_bound}, this happens with probability at most $\delta$.
\end{proof}

\begin{proof}[Proof of Theorem~\ref{thm:abstract_control}]
This follows from Theorem~\ref{thm:abstract_control_general}.
\end{proof}

\begin{proof}[Proof of Proposition \ref{prop:generic}]
Let $G$ denote the CDF of $\Rhat(\lambda)$. If $R(\lambda) > \Rhat^+(\lambda)$, then by definition, $g(\Rhat(\lambda); R(\lambda)) < \delta$. As a result, 
$$
P(R(\lambda) > \Rhat^+(\lambda))\le P(g(\Rhat(\lambda); R(\lambda)) < \delta) \le P(G(\Rhat(\lambda)) < \delta).
$$
Let $G^{-1}(\delta) = \sup\{x: G(x)\le \delta\}$. Then  $$
P(G(\Rhat(\lambda)) < \delta) \le P(\Rhat(\lambda) < G^{-1}(\delta)) \le \delta.
$$
This implies that $P(R(\lambda) > \Rhat^+(\lambda))\le \delta$ and completes the proof.
\end{proof}

\begin{proof}[Proof of Proposition \ref{prop:WSR}]
This proof is essentially a restatement of the proof of Theorem 4 in \cite{waudby2020variance}. We present it here for completeness. Let $\mathcal{K}_i = \mathcal{K}_i(R(\lambda); \lambda)$, $\mathcal{F}_0$ be the trivial sigma-field, and $\mathcal{F}_i$ be the sigma-field generated by $(L_1(\lambda), \ldots, L_i(\lambda))$. Then $\mathcal{F}_0\subset \mathcal{F}_1 \subset \ldots \subset \mathcal{F}_n$ is a filtration. By definition, $\nu_i(\lambda)\in \mathcal{F}_{i-1}$ is a predictable sequence and $\mathcal{K}_i\in \mathcal{F}_i$. Since $\E[L_i(\lambda)] = R(\lambda)$, 
$$
\E[\mathcal{K}_i\mid \mathcal{F}_{i-1}] = \mathcal{K}_{i-1} \E[1 - \nu_i(\lambda)(L_i(\lambda) - R(\lambda))\mid \mathcal{F}_{i-1}] = \mathcal{K}_{i-1}.
$$
In addition, since $\nu_i\in [0, 1]$ and $(L_i(\lambda) - R(\lambda))\in [-1, 1]$, each component $1 - \nu_i(\lambda)(L_i(\lambda) - R(\lambda))\ge 0$. Thus, $\{\mathcal{K}_i: i = 1,\ldots, n\}$ is a non-negative martingale with respect to the filtration $\{\mathcal{F}_i: i = 1, \ldots, n\}$. By Ville's inequality,
$$
P\left(\max_{i = 1, \ldots, n}\mathcal{K}_i \ge \frac{1}{\delta}\right)\le \delta.
$$
On the other hand, since $\nu_i\ge 0$, $\mathcal{K}_i(R; \lambda)$ is increasing in $R$ almost surely for every $i$. By definition of $\Rhat_{\mathrm{WSR}}^+(\lambda)$, if $\Rhat_{\mathrm{WSR}}^+(\lambda) < R(\lambda)$, then $P(\max_{i = 1,\ldots,n}\mathcal{K}_i \ge 1 / \delta)$. Therefore,
$$
P\left(\Rhat_{\mathrm{WSR}}^+(\lambda) < R(\lambda)\right) \le P\left(\max_{i}\mathcal{K}_i \ge \frac{1}{\delta}\right)\le \delta.
$$
This proves that $\Rhat_{\mathrm{WSR}}^+(\lambda)$ is a valid upper confidence bound of $R(\lambda)$.
\end{proof}

\begin{proof}[Proof of Theorem~\ref{thm:clt_coverage}]
Define $\lambda^*$ as in the proof of Theorem~\ref{thm:abstract_control_general}. 
Suppose $R(\lhat^{\textnormal{CLT}}) > \alpha$
By the definition of $\lambda^*$ and the monotonicity and continuity of $R(\cdot)$, this implies
$\lhat^{\textnormal{CLT}} < \lambda^*$. By the definition of $\lhat^{\textnormal{CLT}}$, this further implies that $\Rhat^+(\lambda^*) < \alpha$. But
\begin{equation*}
    \limsup_{n} P(\Rhat^+(\lambda^*) < \alpha) = \delta,
\end{equation*}
by the CLT, which implies the desired result.
\end{proof}


\begin{proof}[Proof of Theorem~\ref{thm:unconditional-optimality}]
Suppose $R(\T') \le R(\T_\lambda)$. Write $\rho_x(y)$ for $\rho_x(y; \emptyset)$. Then,
\begin{equation*}
    \int_{\X} \int_{\T'(x)} \rho_x(y) dy \ dP(x) \ge  \int_{\X} \int_{\T_\lambda(x)} \rho_x(y) dy \ dP(x). 
\end{equation*}
This further implies 
\begin{equation*}
    \int_{\X} \int_{\T'(x) \setminus \T_\lambda(x)} \rho_x(y) dy \ dP(x) \ge  \int_{\X} \int_{\T_\lambda(x) \setminus \T'(x)} \rho_x(y) dy \ dP(x). 
\end{equation*}
For $y \in (\T'(x) \setminus \T_\lambda(x))$, we have $\rho_x(y) < \zeta$, whereas for $y \in (\T_\lambda(x) \setminus \T'(x))$ we have $\rho_x(y) \ge \zeta$. Therefore,
\begin{equation*}
    \int_{\X} \int_{\T'(x) \setminus \T_\lambda(x)} 1 dy \ dP(x) \ge  \int_{\X} \int_{\T_\lambda(x) \setminus \T'(x)} 1 dy \ dP(x), 
\end{equation*}
which implies the desired result.
\end{proof}

\begin{proof}[Proof of Theorem~\ref{thm:unconditional-optimality-ext}]
The proof is similar to that of Theorem~\ref{thm:unconditional-optimality}. If $R(\T') \le R(\T_\lambda)$, then
\begin{align*}
    &\E\left[\E\left[L(Y; \T'(X))\mid X\right]\right] \le \E\left[\E\left[L(Y; \T_\lambda(X))\mid X\right]\right]\\
    \Longrightarrow & \E\left[\E\left[\int_{z\in \T^{'c}(X)}\ell(Y; z)d\mu(z)\mid X\right]\right]\le \E\left[\E\left[\int_{z\in \T_\lambda^{c}(X)}\ell(Y; z)d\mu(z)\mid X\right]\right]\\
    \Longrightarrow & \E\left[\E\left[\int_{z\in \T'(X)}\ell(Y; z)d\mu(z) \mid X\right]\right]\ge \E\left[\E\left[\int_{z\in \T_\lambda(X)}\ell(Y; z)d\mu(z)\mid X\right]\right]\\
    \Longrightarrow & \E\left[\int_{z\in \T'(X)}\E\left[\ell(Y; z) \mid X\right]d\mu(z)\right]\ge \E\left[\int_{z\in \T_\lambda(X)}\E\left[\ell(Y; z)\mid X\right]d\mu(z)\right]\\
    \Longrightarrow & \E\left[\int_{z\in \T'(X)\setminus \T_\lambda(X)}\E\left[\ell(Y; z) \mid X\right]d\mu(z)\right]\ge \E\left[\int_{z\in \T_\lambda(X) \setminus \T'(X)}\E\left[\ell(Y; z)\mid X\right]d\mu(z)\right]\\
    \Longrightarrow & \E\left[\int_{z\in \T'(X)\setminus \T_\lambda(X)}-\lambda d\mu(z)\right]\ge \E\left[\int_{z\in \T_\lambda(X) \setminus \T'(X)}-\lambda d\mu(z)\right]\\
    \Longrightarrow & \E[|\T'(X)\setminus \T_\lambda(X)|]\ge \E[|\T_\lambda(X)\setminus \T'(X)|]\\
    \Longrightarrow & \E[|\T'(X)|]\ge \E[|\T_\lambda(X)|].
\end{align*}
\end{proof}

\begin{proof}[Proof of Proposition \ref{prop:ustat2_ucb}]
Let $Z_i = (X_i, Y_i)$ and $\phi(Z_i, Z_j) = L(Y_i, Y_j, \T_\lambda(X_i, X_j))$. First, we apply a representation of U-statistics due to \cite{hoeffding1963} that shows many tail inequalities for sums of i.i.d.\ random variables hold for U-statistics of order two with an effective sample size $\lfloor n / 2\rfloor$. Specificially, let $m = \lfloor n / 2\rfloor$ and $\pi: \{1, \ldots, n\}\mapsto \{1, \ldots, n\}$ be a uniform random permutation. For each $\pi$, define 
 \[\Rhat_{\pi}(\lambda) = \frac{1}{m}\sum_{j=1}^{m}\phi\left( Z_{\pi(2j-1)}, Z_{\pi(2j)}\right).\]
 Note that the summands in $\Rhat_{\pi}(\lambda)$ are independent given $\pi$. Then it is not hard to see that
 \[\Rhat(\lambda) = \E_{\pi}[\Rhat_{\pi}(\lambda)],\]
where $\E_{\pi}$ denotes the expectation with respect to $\pi$ while conditioning on $Z_1, \ldots, Z_n$. By Jensen's inequality, for any convex function $\psi$,
\[\E[\psi(\Rhat(\lambda))] = \E[\phi(\E_{\pi}[\Rhat_{\pi}(\lambda)])]\le \E[\E_{\pi}\psi(\Rhat_{\pi}(\lambda))] = \E_{\pi}[\E \psi(\Rhat_{\pi}(\lambda))].\]
Since $\Rhat_{\pi}(\lambda)$ has identical distributions for all $\pi$,
\begin{equation}
  \label{eq:hoeffding_representation}
  \E[\psi(\Rhat(\lambda))] = \E[\psi(\Rhat_{\mathrm{id}}(\lambda))]
\end{equation}
where $\mathrm{id}$ is the permutation that maps each element to itself. 

For sums of i.i.d.\ random variables, the Hoeffding's inequality (Proposition \ref{prop:hoeffding}) is derived by setting $\psi(z) = \exp\{\nu z\}$ \citep{hoeffding1963}, and the Bentkus inequality (Proposition \ref{prop:bentkus}) is derived by setting $\psi(z) = (z - \nu)_{+}$. Therefore, the same tail probability bounds hold for $\Rhat_{\mathrm{id}}(\lambda)$ and thus $\Rhat(\lambda)$ by \eqref{eq:hoeffding_representation}. This proves the first two bounds. 

To prove the third bound, we apply the technique of \cite{maurer2006concentration} on self-bounding functions of iid random variables. Write $\Rhat(\lambda)$ as $U(Z_1, \ldots, Z_n)$ and let 
\[U_{i} = \inf_{z_i}U(Z_1, \ldots, Z_{i-1}, z_{i}, Z_{i+1}, \ldots, Z_{n}).\]
Note that $U_{i}$ is independent of $Z_{i}$. Since $\phi(\cdot)\ge 0$, we have
\[0\le U - U_{i} \le \frac{2}{n(n-1)}\sum_{i\not=j}\phi(Z_{i}, Z_{j}).\]
Since $\phi(Z_i, Z_j)\le 1$,
\[\frac{n}{2}(U - U_{i})\le 1,\]
and
\begin{align*}
  \sum_{i=1}^{n}(U - U_{i}) \le \frac{2}{n(n-1)}\sum_{i=1}^{n}\sum_{i\not=j}\phi(Z_{i}, Z_{j}) = 2U.
\end{align*}
If we let $W = (n / 2)U$ and $W_i = (n / 2)U_i$, then
\begin{equation}
  \label{eq:W_self_bound}
  W - W_{i}\le 1, \quad \sum_{i=1}^{n}(W - W_{i})^2\le 2W.
\end{equation}
In the proof of Theorem 13, \cite{maurer2006concentration} shows that for any $\nu > 0$, 
  \[\log \E[\exp\{\nu(\E[W] - W)\}]\le \frac{2\nu G(\nu)}{1 + 2 G(\nu)} \E [W].\]
  By Markov's inequality, for any $t \in (0, \E[U])$,
  \begin{align*}
  & P\left( U \le t\right)  = P\left( \E[W] - W \ge \E[W] - \frac{n}{2}t\right) \\
  & \le \exp\left\{\min_{\nu > 0}\nu\left(-\E[W] + \frac{n}{2}t + \frac{2 G(\nu)}{1 + 2 G(\nu)} \E [W]\right)\right\}\\
  & = \exp\left\{\min_{\nu > 0}\frac{n\nu}{2} \left( t - \frac{1}{1 + 2 G(\nu)} \E [U]\right)\right\}.
  \end{align*}
The proof is completed by replacing $U$ by $\Rhat(\lambda)$ and $\E[U]$ by $R(\lambda)$.
\end{proof}

\begin{prop}[Impossibility of valid UCB for unbounded losses in finite samples]\label{prop:impossibility}
Let $\mathcal{F}$ be the class of all distributions supported on $[0, \infty)$ with finite mean, and $\mu(F)$ be the mean of the distribution $F$. Let $\hat{\mu}^+$ be any function of $Z_1, \ldots, Z_n\stackrel{\mathrm{i.i.d.}}{\sim} F$ such that $P(\hat{\mu}^{+} \ge \mu(F))\ge 1 - \delta$ for any $n$ and $F\in \mathcal{F}$. Then $P(\hat{\mu}^{+} = \infty)\ge 1 - \delta$.
\end{prop}
\begin{proof}[Proof of Proposition \ref{prop:impossibility}]
It is clear that $\mathcal{F}$ satisfies the conditions (i), (ii), and (iii) in \cite{bahadur1956}. For any such $\hat{\mu}^{+}$, $[0, \hat{\mu}^{+}]$ is a $(1 - \delta)$ confidence interval of $\mu(F)$. By their Corollary 2, we know that for any $\mu\in \{\mu(F): F\in\mathcal{F}\}$ and $F\in \mathcal{F}$
\[P_{F}(\mu\in [0, \hat{\mu}^{+}])\ge 1- \delta\Longleftrightarrow P_{F}(\mu \le \hat{\mu}^{+})\ge 1 - \delta.\]
The proof is completed by letting $\mu\rightarrow \infty$.
\end{proof}

\begin{proof}[Proof of Theorem~\ref{thm:ranking}]
This follows from Theorem~\ref{thm:abstract_control_general}
\end{proof}

\begin{proof}[Proof of Theorem~\ref{thm:metric_learning}]
This follows from Theorem~\ref{thm:abstract_control_general}
\end{proof}

\section{An Exact Bound for Binary Loss}
\label{app:binary_loss}

When the loss takes values in $\{0,1\}$, for a fixed $\lambda$ the loss at each point is a Bernoulli random variable, and the risk is simply the mean of this random variable. In this case, we can give a tight upper confidence bound by simple extracting the relevant quantile of a binomial distribution; see \cite{brown2001interval} for other exact or approximate upper confidence bounds. Explicitly, we have
\begin{equation}
P\left( \widehat{R}(\lambda) \le t\right) = P\left({\rm Binom}(n,R(\lambda)) \le \lceil nt \rceil\right),
\end{equation}
which is the same expression as in the Bentkus bound, improved by a factor of $e$. From this, we obtain a lower tail probability bound for $\Rhat(\lambda)$:
\begin{align}
g^{\mathrm{bin}}(t; R(\lambda)) \triangleq P\left({\rm Binom}(n,R(\lambda)) \le \lceil nt \rceil\right).
\end{align}
By Proposition \ref{prop:generic}, we obtain a $(1 - \delta)$ upper confidence bound for $R(\lambda)$ as 
\begin{equation*}
\label{eq:Rhat_binary}
\Rhat_{\mathrm{bin}}^+(\lambda) = \sup\Big\{R: g^{\mathrm{bin}}(\Rhat(\lambda); R)\ge \delta\Big\}.
\end{equation*}
We obtain $\lhat^{{\rm bin}}$ by inverting the above bound computationally, yielding the following corollary:
\begin{theorem}[RCPS for binary variables]\label{thm:binary}
    In the setting of Theorem~\ref{thm:abstract_control}, assume additionally that the loss takes values in \{0,1\}. Then, $\T_{\lhat^{{\rm bin}}}$ is a $(\alpha,\delta)$-RCPS.
\end{theorem}
The binary loss case case results in a classical tolerance region, as discussed previously in \cite{vovk2012conditional} and \cite{Park2020PAC}.

\section{Conformal Calibration}
In the special case where the the loss function $L$ takes values only in $\{0,1\}$, 
it is also possible to select $\hat{\lambda}$ to control the quantity $\E[R(T_{\hat{\lambda}})]$ below a desired level $\alpha$. When $\Y' = 2^\Y$ and  $L(Y_i,\T_\lambda(X_i)) = \ind{\Y_i \notin \T_\lambda(X_i)}$,  this is the well-known case of split conformal prediction; see \cite{gupta2020nested}. To the best of our knowledge, the general case where $\Y' \ne 2^\Y$ has not been explicitly dealt with, so we record this mild generalization here.

For $i = 1,\dots,n$, we define the following score:
\begin{equation*}
s_i := \min\{\lambda \in \Lambda : L(Y_i, \T_\lambda(X_i)) = 0\}, 
\end{equation*}
where we assume that the family of sets $\T_\lambda$ is such that the minimal element exists with probability one. (This is always true in practice, where $\Lambda$ is finite.) For a fixed risk level $\alpha \in (0,1)$, we then choose the threshold as follows: 
\begin{equation*}
\hat{\lambda} = \frac{n+1}{n}(1-\alpha)\text{ empirical quantile of } \{s_i : i = 1, \dots, n\}.
\end{equation*}
We then have the following risk-control guarantee:
\begin{prop}[Validity of conformal calibration]
\label{prop:coverage}
In the setting above,
\begin{equation*}
\E[R (\T_{\hat{\lambda}})] \le \alpha.
\end{equation*}
\end{prop}
This result follows from the usual conformal prediction exchangeability proof; see, e.g., \cite{romano2019conformalized}.

\section{Further Comparisons of Upper Confidence Bounds}
\label{app:additional_plots}

We present additional plots comparing the upper confidence bounds with $\delta = 0.01$ and $\delta = 0.001$. The counterparts of Figure \ref{fig:UCB_bounded_delta0.1} for bounded cases are presented in Figure~\ref{fig:UCB_bounded_delta0.01} and \ref{fig:UCB_bounded_delta0.001}, and the counterparts of Figure \ref{fig:UCB_unbounded_delta0.1} for unbounded cases are presented in Figure \ref{fig:UCB_unbounded_delta0.01} and \ref{fig:UCB_unbounded_delta0.001}.

To further compare the HB bound and WSR bound for the binary loss case, in Figure \ref{fig:UCB_bounded_bestprob} we present the fraction of samples on which the HB bound or the WSR bound is the winner among the four bounds, excluding the CLT bound due to the undercoverage. The HB bound is more likely to be tighter than the WSR bound, especially when the mean $\mu$ or the level $\delta$ is small. Moreover, the symmetry between two curves in each panel is due to the fact that the simple Hoeffding bound and empirical Bernstein bound never win. These results show that the WSR better is not uniformly better than the HB bound, although it is still the best all-around choice for bounded losses.

\begin{figure}[H]
\begin{subfigure}{0.45\textwidth}
\includegraphics[width = 0.98\textwidth]{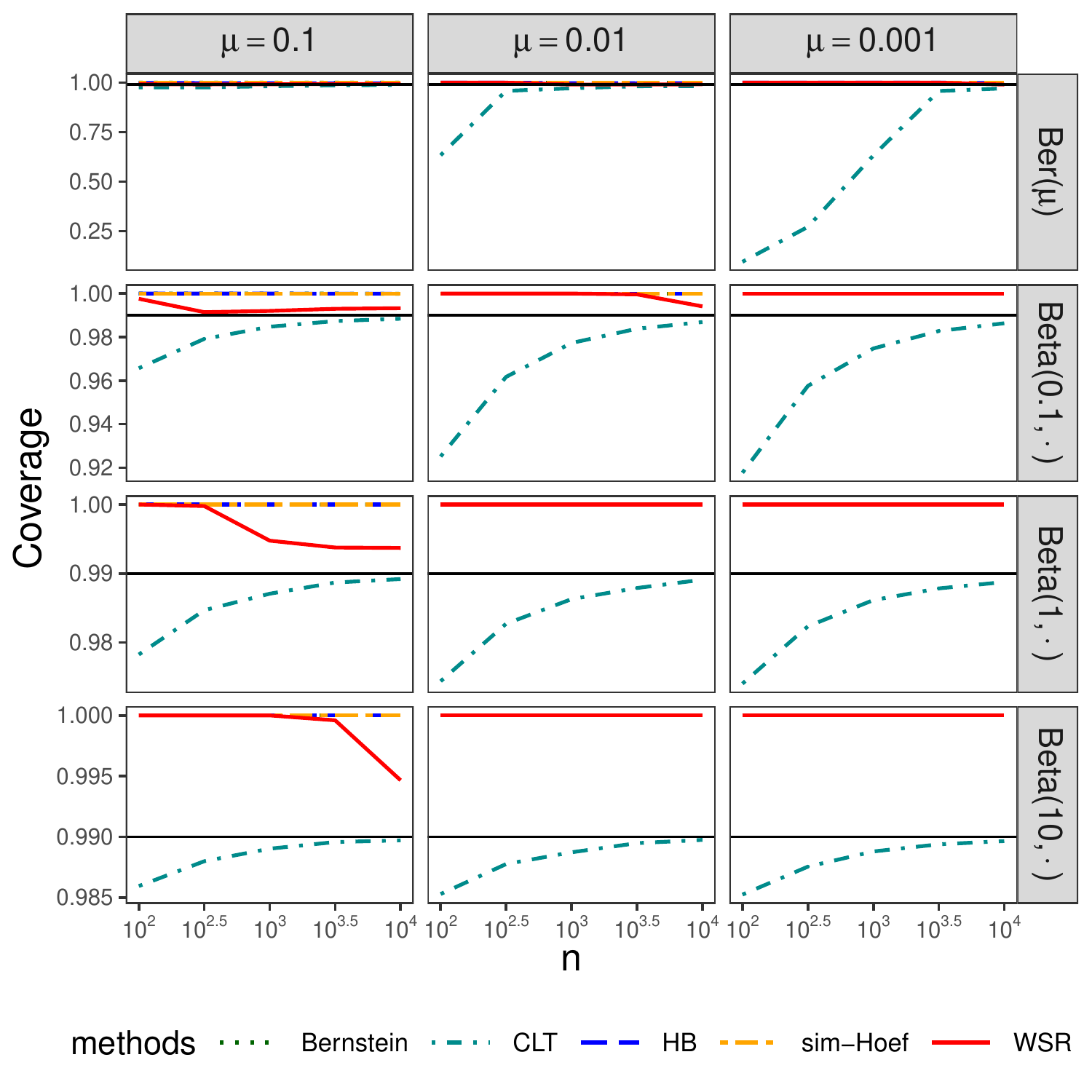}
\caption{Coverage $P(\Rhat(\lambda) \ge R(\lambda))$}\label{subfig:UCB_bounded_coverage_delta0.01}
\end{subfigure}
\begin{subfigure}{0.45\textwidth}
\includegraphics[width = 0.98\textwidth]{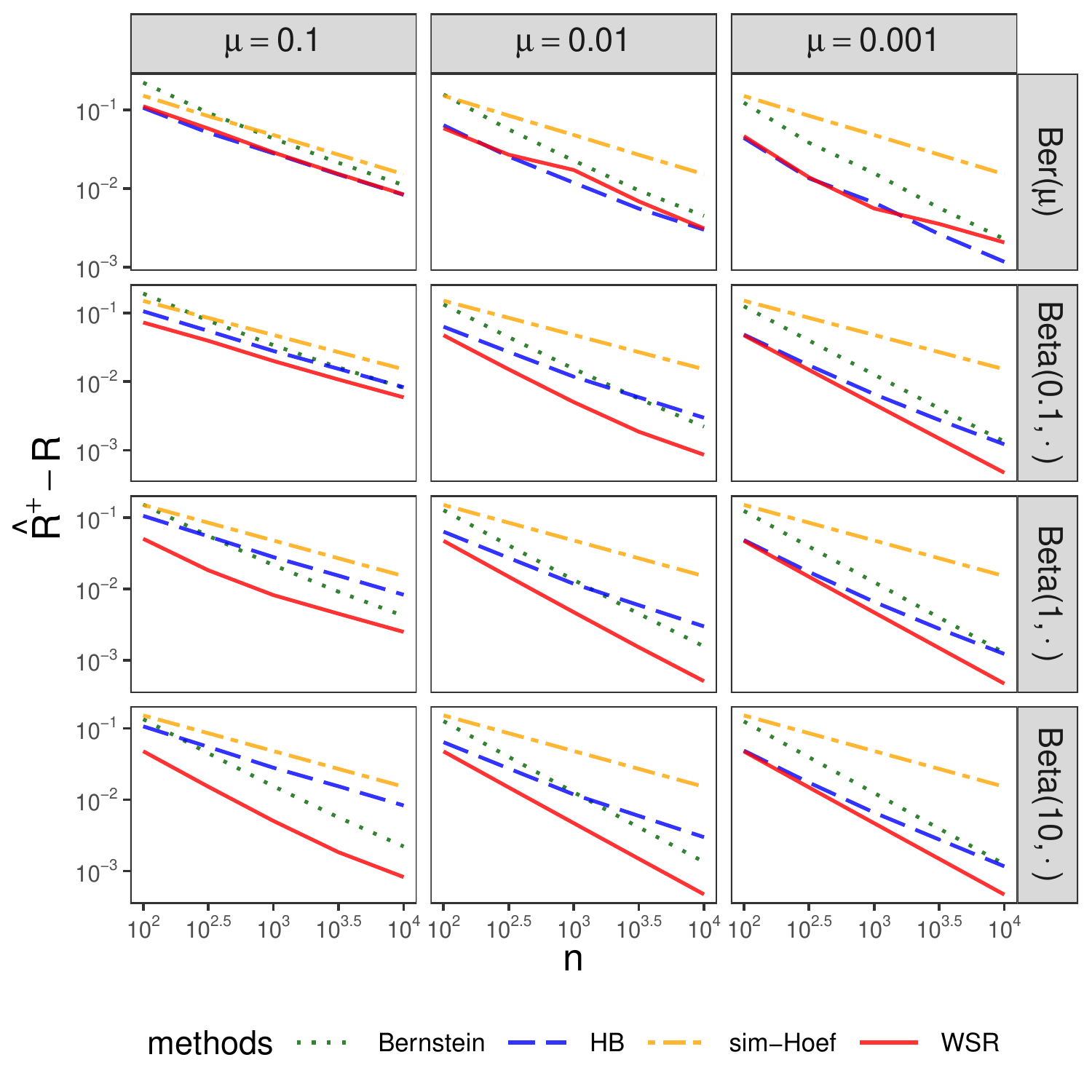}
\caption{Median of $\Rhat(\lambda) - R(\lambda)$}\label{subfig:UCB_bounded_medgap_delta0.01}
\end{subfigure}
\caption{Numerical evaluations of the simple Hoeffding bound \eqref{eq:Rhat_sim_hoeffding}, HB bound \eqref{eq:Rhat_HB}, empirical Bernstein bound \eqref{eq:eBern}, CLT bound \eqref{eq:Rhat_CLT}, and WSR bound (Proposition \ref{prop:WSR}) on a million independent samples of size $n$ with $\delta = 0.01$. Each row corresponds to a type of distribution and each column corresponds to a value of the mean. The CLT bound is excluded in (b) because it does not achieve the target coverage in most of the cases.}\label{fig:UCB_bounded_delta0.01}
\end{figure}

\begin{figure}
\begin{subfigure}{0.45\textwidth}
\includegraphics[width = 0.98\textwidth]{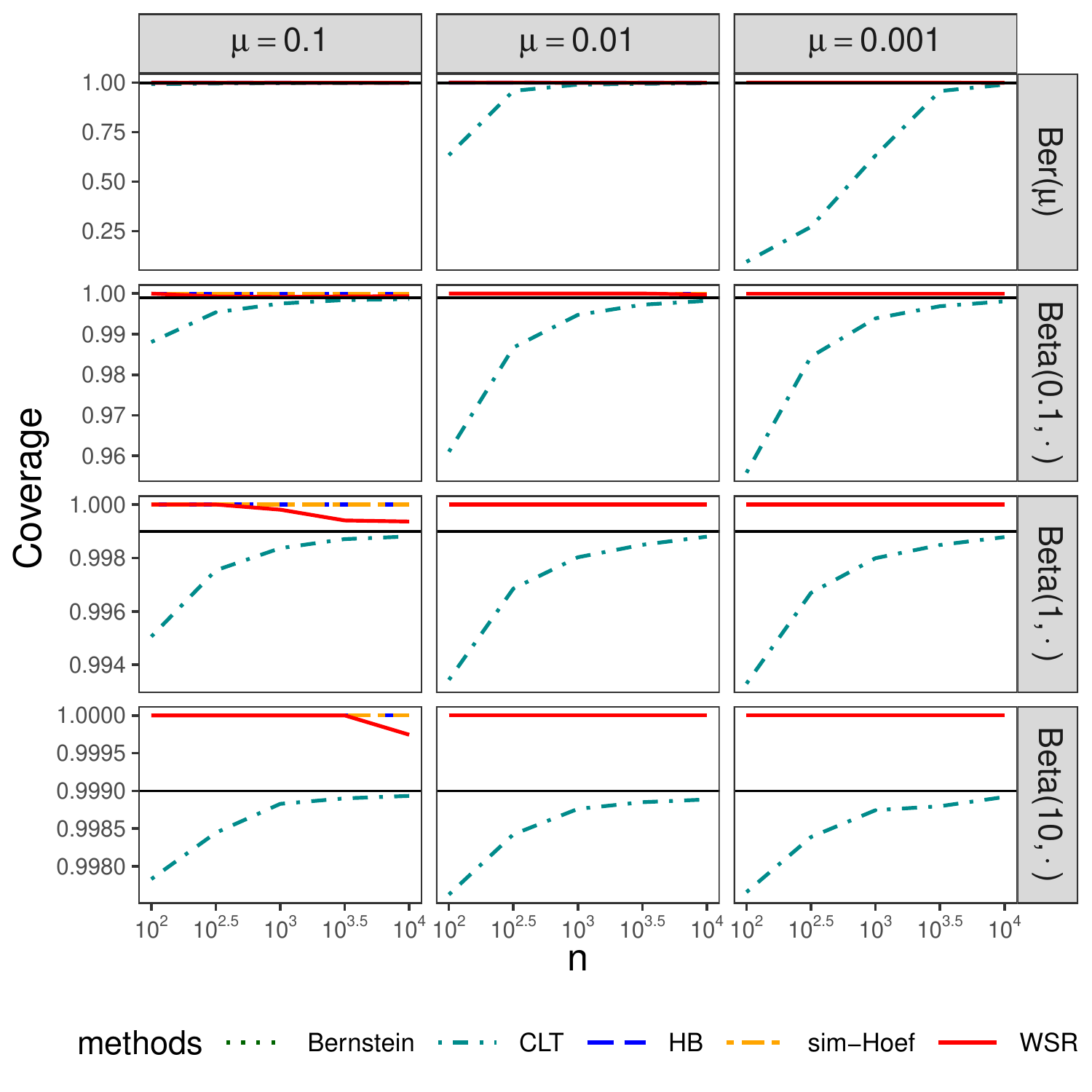}
\caption{Coverage $P(\Rhat(\lambda) \ge R(\lambda))$}\label{subfig:UCB_bounded_coverage_delta0.001}
\end{subfigure}
\begin{subfigure}{0.45\textwidth}
\includegraphics[width = 0.98\textwidth]{figures/UCB_bounded_medgap_delta0_1.pdf}
\caption{Median of $\Rhat(\lambda) - R(\lambda)$}\label{subfig:UCB_bounded_medgap_delta0.001}
\end{subfigure}
\caption{Numerical evaluations of the simple Hoeffding bound \eqref{eq:Rhat_sim_hoeffding}, HB bound \eqref{eq:Rhat_HB}, empirical Bernstein bound \eqref{eq:eBern}, CLT bound \eqref{eq:Rhat_CLT}, and WSR bound (Proposition \ref{prop:WSR}) on a million independent samples of size $n$ with $\delta = 0.001$. Each row corresponds to a type of distribution and each column corresponds to a value of the mean. The CLT bound is excluded in (b) because it does not achieve the target coverage in most of the cases.}\label{fig:UCB_bounded_delta0.001}
\end{figure}

\begin{figure}
\begin{subfigure}{0.49\textwidth}
\includegraphics[width = 0.98\textwidth]{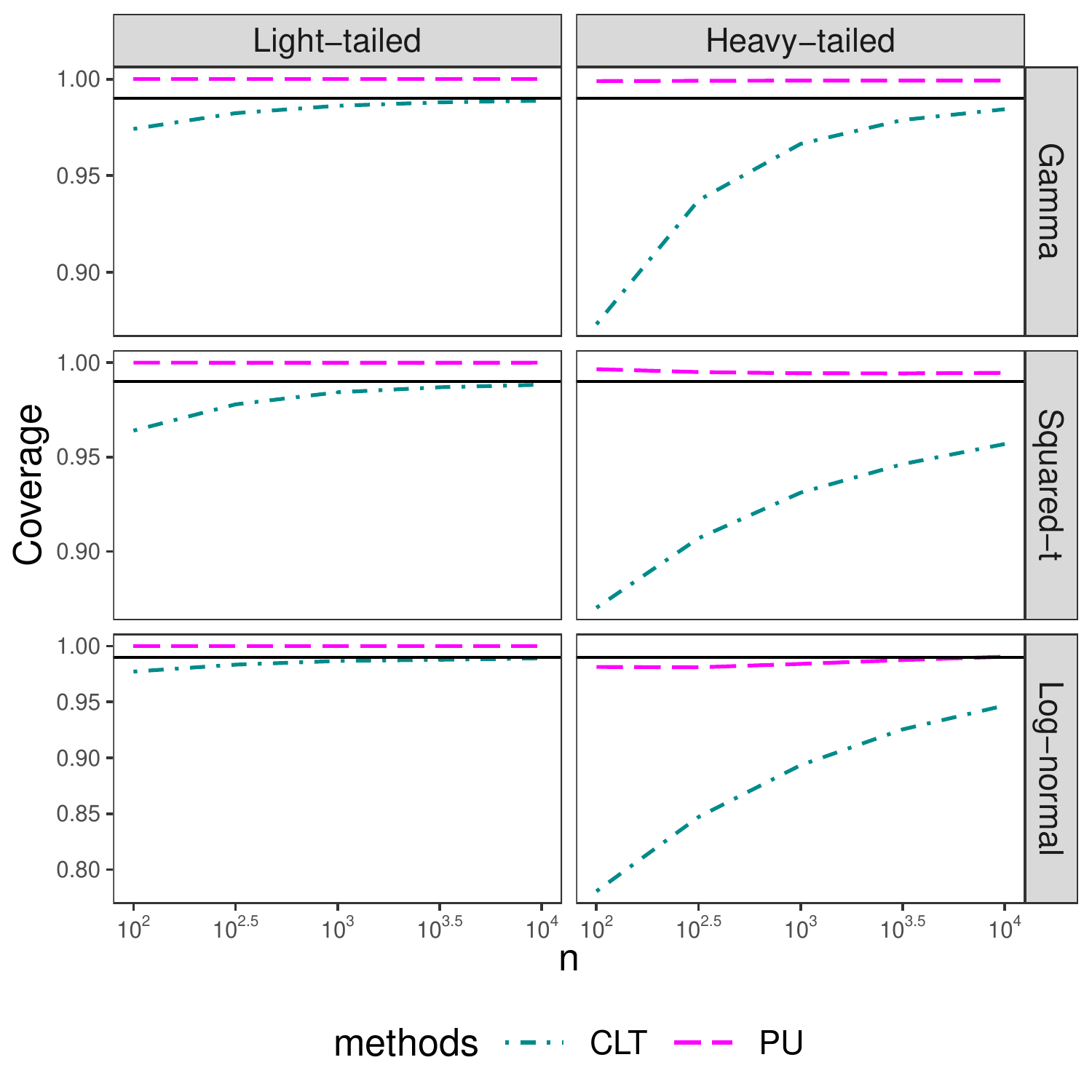}
\caption{Coverage $P(\Rhat(\lambda) \ge R(\lambda))$}\label{subfig:UCB_unbounded_coverage_delta0.01}
\end{subfigure}
\begin{subfigure}{0.49\textwidth}
\includegraphics[width = 0.98\textwidth]{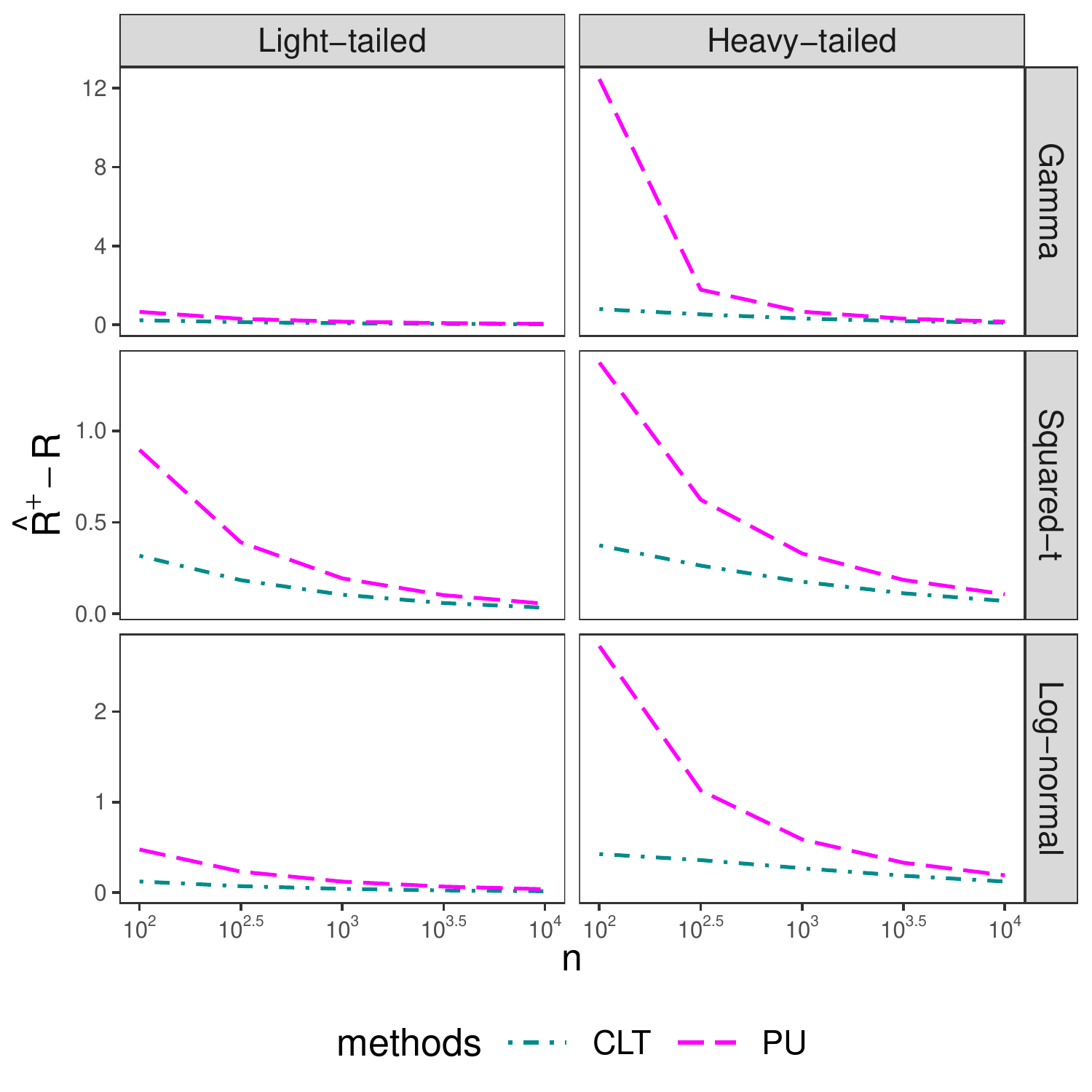}
\caption{Median of $\Rhat(\lambda) - R(\lambda)$}\label{subfig:UCB_unbounded_medgap_delta0.01}
\end{subfigure}
\caption{Numerical evaluations of the PU bound \eqref{eq:Rhat_PU} with the estimated coefficient of variation and the CLT bound \eqref{eq:Rhat_CLT}, on a million independent samples of size $n$ from each distribution in Table \ref{tab:unbounded} with $\delta = 0.01$. Each row corresponds to a type of distribution and each column corresponds to a value of the mean. }\label{fig:UCB_unbounded_delta0.01}
\end{figure}

\begin{figure}
\begin{subfigure}{0.49\textwidth}
\includegraphics[width = 0.98\textwidth]{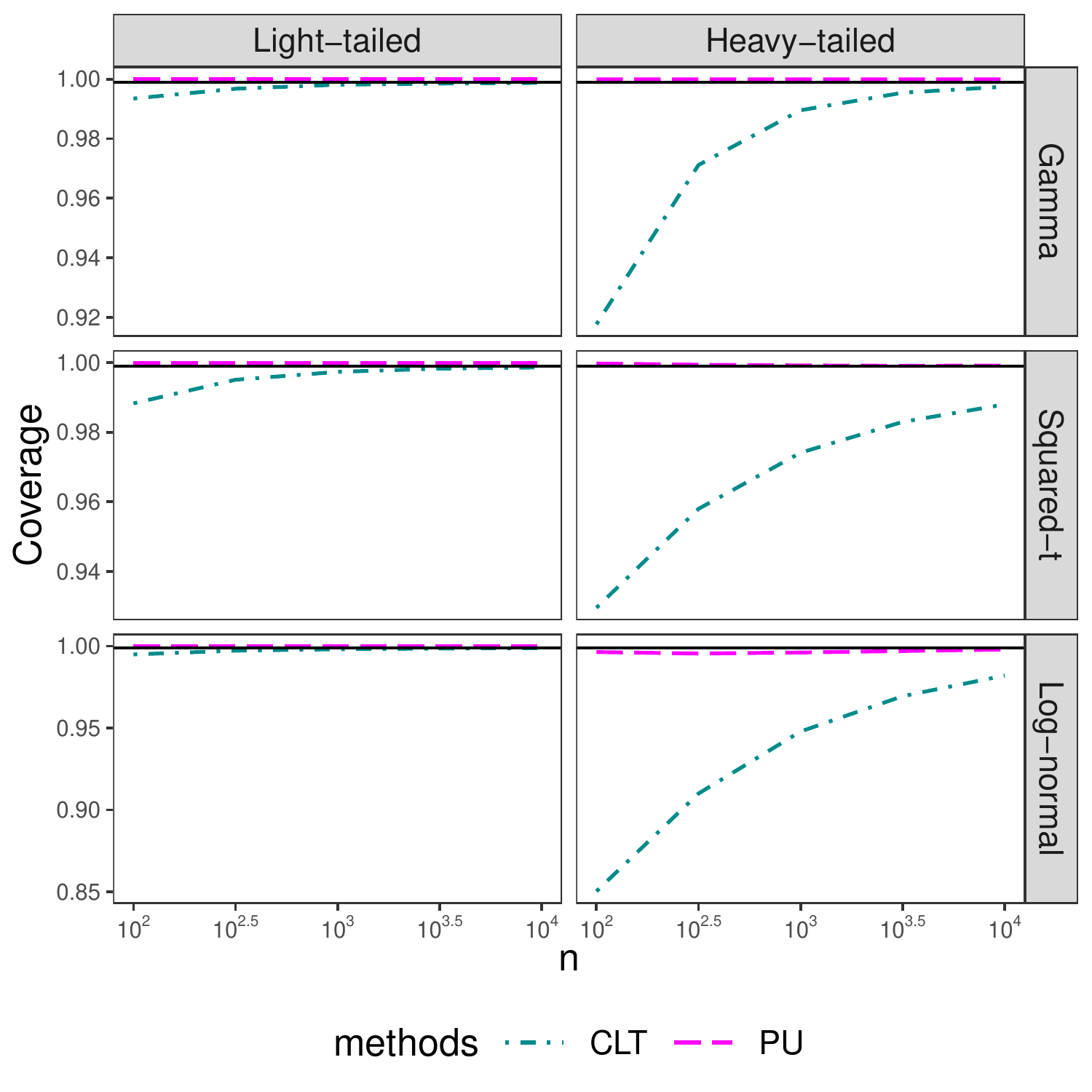}
\caption{Coverage $P(\Rhat(\lambda) \ge R(\lambda))$}\label{subfig:UCB_unbounded_coverage_delta0.001}
\end{subfigure}
\begin{subfigure}{0.49\textwidth}
\includegraphics[width = 0.98\textwidth]{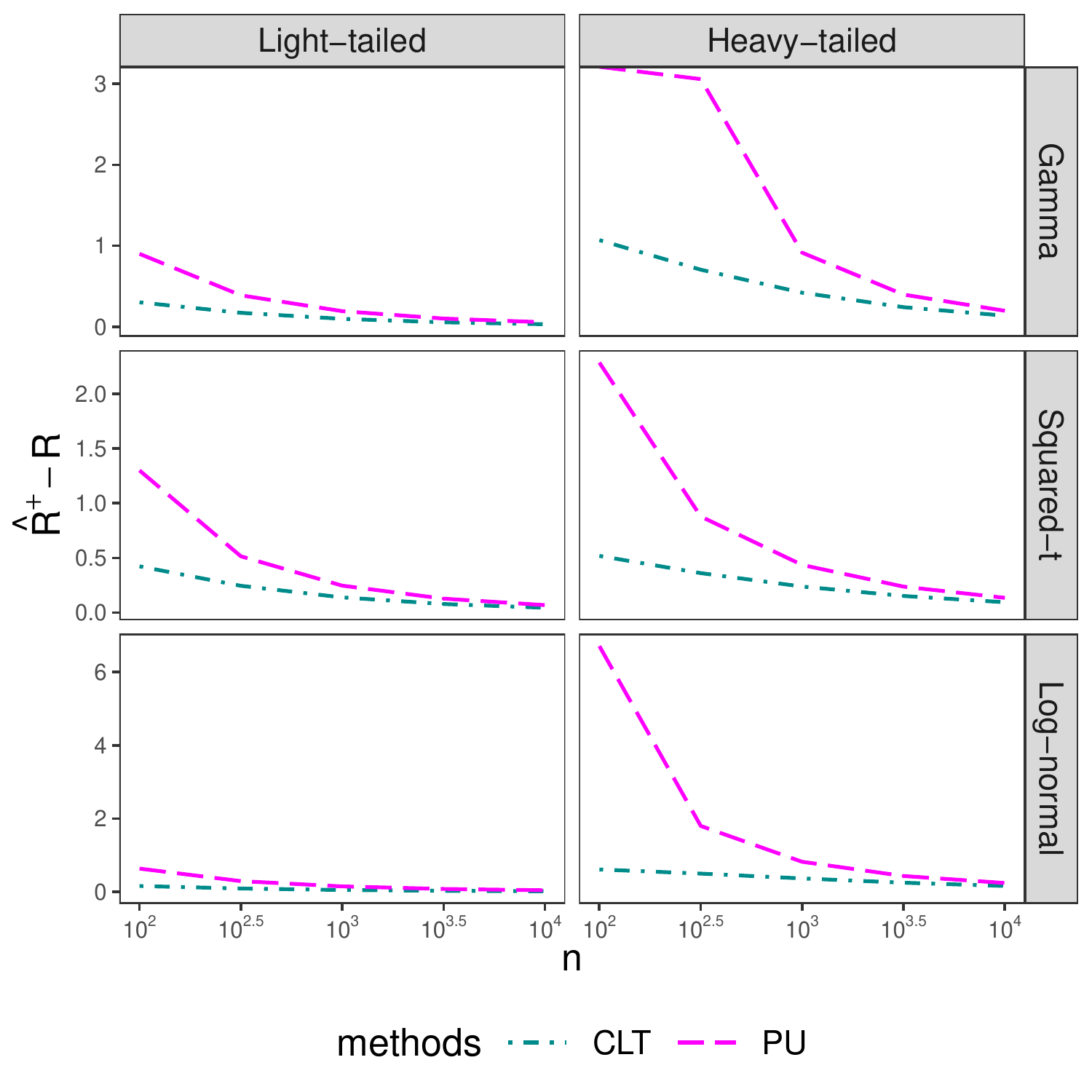}
\caption{Median of $\Rhat(\lambda) - R(\lambda)$}\label{subfig:UCB_unbounded_medgap_delta0.001}
\end{subfigure}
\caption{Numerical evaluations of the PU bound \eqref{eq:Rhat_PU} with the estimated coefficient of variation and the CLT bound \eqref{eq:Rhat_CLT},  on a million independent samples of size $n$ from each distribution in Table \ref{tab:unbounded} with $\delta = 0.001$. Each row corresponds to a type of distribution and each column corresponds to a value of the mean. }\label{fig:UCB_unbounded_delta0.001}
\end{figure}


\begin{figure}
    \centering
    \includegraphics[width = 0.5\textwidth]{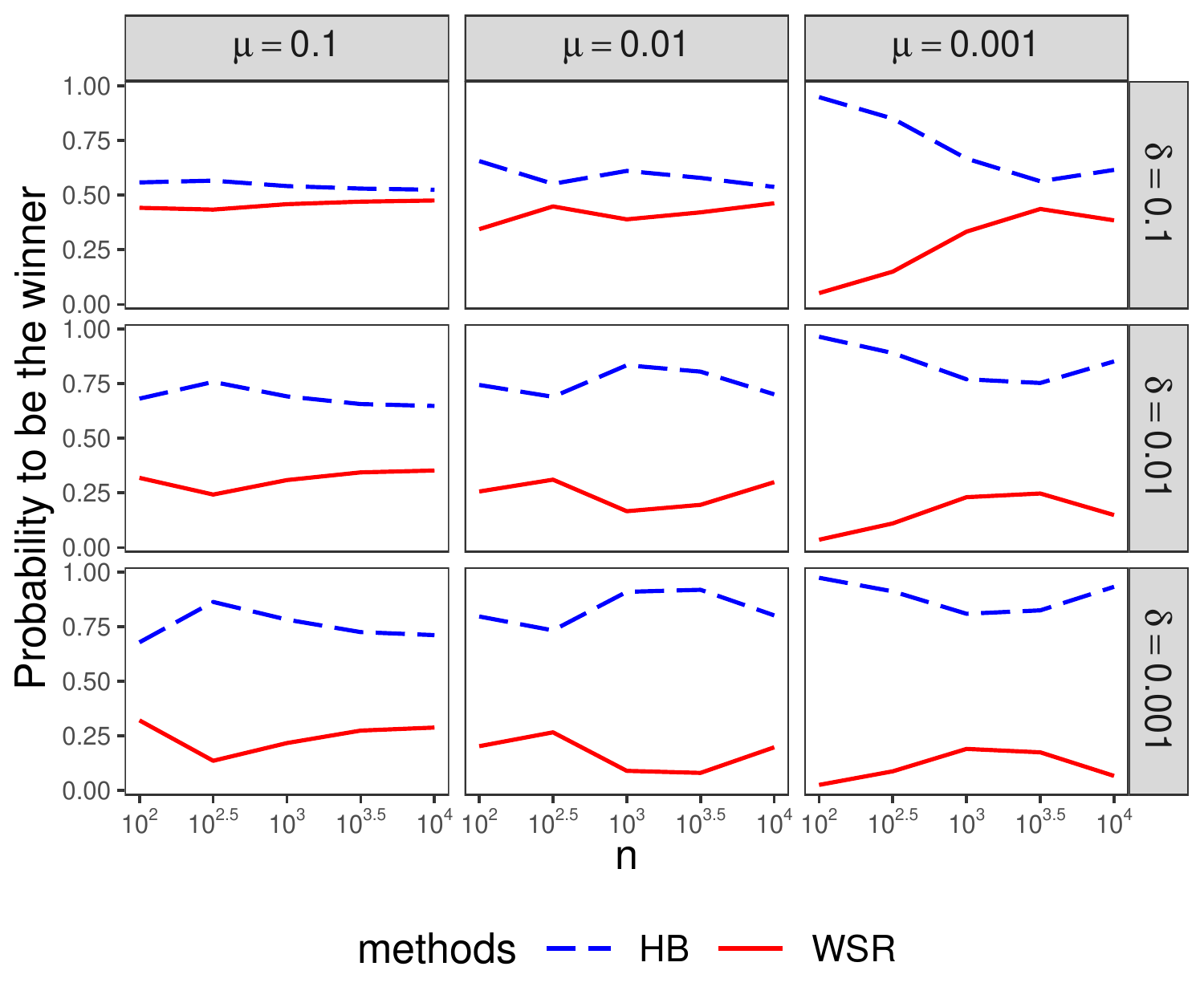}
    \caption{Fraction of samples on which the HB bound or the WSR bound is the winner among the four bounds, excluding the CLT bound, for Bernoulli distributions. Each row corresponds to a level and each column corresponds to a value of mean.}
    \label{fig:UCB_bounded_bestprob}
\end{figure}

\clearpage
\section{Adaptive Score Renormalization for Polyp Segmentation}
\label{app:polyp-renorm}
This section describes in detail the construction of our predictor in the polyp segmentation example in Section~\ref{sec:segmentation}. In order to construct a good set predictor from the raw predictor, we draw on techniques from the classical literature on image processing to detect and emphasize local peaks in the raw scores. In particular, we construct a renormalization function $r : [0,1]^{m \times n} \to [0,1]^{m \times n}$, which is a composition of a set of morphological operations.
We will now list a set of operations whose composition will define $r$.

Define the discrete Gaussian blur operator as $g: [0,1]^{m \times n} \times \R_{++} \times \mathbb{O}_+ \to [0,1]^{m \times n}$, where $\mathbb{O}_+$ is the set of odd numbers.
The second argument to $g$ is the standard deviation $\sigma$ of a Gaussian kernel in pixels and $k$ is the side length of the kernel in pixels.
The Gaussian kernel is then the matrix 
\begin{equation}
    K(\sigma,k)_{i,j}= C\exp\left\{-\frac{1}{2\sigma^2}\Big\|[i,j]-[\ceil{k/2},\ceil{k/2}]\Big\|^2\right\},
\end{equation}
where $C$ is chosen such that $\sum_{i,j}K_{i,j}=1$.
The function $g$ then becomes $g(S,\sigma,k)=S * K(\sigma,k)$, where $*$ denotes the 2D convolution operator.

We borrow a technique from mathematical morphology called {\em reconstruction by dilation} and use it to separate local score peaks from their background.
We point the reader to Robinson and Whelan \citep{robinson2004efficient} for an involved description of the algorithm we applied in our codebase.
For the purposes of this paper, we write the reconstruction by dilation algorithm as $dil : [0,1]^{m \times n} \to [0,1]^{m \times n}$.
The output of $dil$ is an array containing only the local peaks from the input, with all other areas set to zero.

Define the binarization function $bin_t: [0,1]^{m \times n} \to \{0,1\}^{m \times n}$ as $bin(x)_{i,j}=\ind{x_{i,j}>t}$.

In the next step, we binarize the local peaks and then split them into disjoint regions through the $2$-connected-components function $conn:\{0,1\}^{m \times n} \to 2^{\{0,1\}^{m \times n}}$.
Viewing a binary matrix $M$ as a graph, we can express it as an adjacency matrix $A \in mn \times mn$ where 
\begin{align}
    A(M)_{i,j} = \mathbbm{1}
    \Bigg\{
        \Big|\Big|\big[ \lfloor i/n \rfloor, \mod (i,n) \big]-\big[ \lfloor j/n \rfloor, \mod (j,n) \big]\Big|\Big| & < 2 \\
        M_{\lfloor i/n \rfloor,\mod (i,n)} & = 1  \\
        M_{\lfloor j/n \rfloor,\mod (j,n)} & = 1 \Bigg\}.
\end{align}
In words, each entry of $A$ corresponds to a pixel, and two pixels are connected by an edge if and only if they are adjacent with entry $1$ in the matrix $M$.
We can use $A$ to define a function $isconnected : \{0,1\}^{m \times n} \times m \times n \times m \times n \to \{0,1\}$ that takes a binary matrix $M$ and two coordinates $(i,j)$ and $(i', j')$ and returns $1$ if the coordinates are connected by a path.
Explicitly, $isconnected = \mathbbm{1}\big\{ \exists k \; : \; A^k_{ni+j,ni'+j'} = 1\big\}$.
Since $isconnected$ is reflexive, symmetric, and transitive, it defines an equivalence relation $\sim$.
We can formally define the set of all equivalence classes over indexes,
\begin{equation}
    \mathcal{E}(A) = \Big\{ \big\{ (i,j) \in m \times n : (i,j) \sim (i',j') \big\} : (i', j') \in m \times n\Big\}.
\end{equation}
Using $\mathcal{E}$, we can draw bounding boxes around each object as
\begin{align}
    bboxes(\mathcal{E}) = \bigg\{ & \Big[ \inf\big\{i : (i,j) \in E \textnormal{ for some } j\big\}, \sup\big\{i : (i,j) \in E \textnormal{ for some } j\big\}\Big] \times \\ 
    & \Big[ \inf\big\{j : (i,j) \in E \textnormal{ for some } i\big\}, \sup\big\{j : (i,j) \in E \textnormal{ for some } i\big\}\Big] : E \in \mathcal{E} \bigg\}
\end{align}
We can proceed to define a function $renorm$ that takes in a matrix of scores $M$ and a set of bounding boxes $bboxes$ and returns a renormalized matrix of scores:

\begin{equation}
    renorm(M,bboxes)_{i,j} = \frac{M_{i,j}}{\underset{b \in bboxes}{\min} \; \underset{\substack{(i',j')\in b \\ (i,j) \in b}}{\max} M_{i',j'}}.
\end{equation}

We can finally define $r$ as $r(M) = renorm(M,bboxes(\mathcal{E}(A(bin_t(g(M,\sigma,k))))))$ for use in Equation~\ref{eq:polyp-sets}.
\end{document}